\documentclass{article}

\usepackage{microtype}
\usepackage{graphicx}
\usepackage{subcaption}
\usepackage{booktabs} 

\usepackage{hyperref}

\usepackage[accepted]{icml2023}

\usepackage{amsmath}
\usepackage{amssymb}
\usepackage{mathtools}
\usepackage{amsthm}
\usepackage{hyperref}       
\usepackage{url}            
\usepackage{booktabs}       
\usepackage{amsfonts}       
\usepackage{nicefrac}       
\usepackage{microtype}      
\usepackage{xcolor}         
\usepackage{graphicx}
\usepackage{dsfont}
\usepackage{thmtools, thm-restate}

\usepackage{multirow}
\usepackage{xspace}
\newcommand{\modelname}{DDGroup\xspace}
\usepackage{enumitem}

\newcommand{\blockcomment}[1]{}
\renewcommand{\l}{\left}
\renewcommand{\r}{\right}

\DeclareMathOperator*{\argmin}{argmin}
\DeclareMathOperator*{\argmax}{argmax}
\DeclareMathOperator*{\var}{\mathbb{V}}
\DeclareMathOperator*{\vol}{vol}
\newcommand{\nn}{\nonumber}

\newcommand{\T}{\top}
\newcommand{\iid}{\stackrel{\tiny{\mathrm{iid}}}{\sim}}

\newcommand{\E}{\mathbb{E}}
\newcommand{\R}{\mathbb{R}}
\renewcommand{\P}{\mathbb{P}}
\newcommand{\I}{\mathds{1}}
\newcommand{\V}{\mathbb{V}}

\newcommand{\cB}{\mathcal{B}}

\newcommand{\cF}{\mathcal{F}}

\newcommand{\cN}{\mathcal{N}}

\newcommand{\cP}{\mathcal{P}}
\newcommand{\cX}{\mathcal{X}}
\newcommand{\cY}{\mathcal{Y}}
\newcommand{\cZ}{\mathcal{Z}}

\newcommand{\tX}{\widetilde{X}}
\newcommand{\tY}{\widetilde{Y}}

\renewcommand{\a}{\alpha}
\renewcommand{\b}{\beta}
\renewcommand{\d}{\delta}

\newcommand{\e}{\varepsilon}
\newcommand{\g}{\gamma}

\newcommand{\s}{\sigma}

\newcommand{\ein}{\e_\mathrm{in}}
\renewcommand{\sin}{\s_\mathrm{in}}
\newcommand{\eout}{\e_\mathrm{out}}
\newcommand{\sout}{\s_\mathrm{out}}
\newcommand{\smax}{\s_\mathrm{max}}
\newcommand{\smin}{\s_{\mathrm{min}}}

\usepackage[capitalize,noabbrev]{cleveref}

\theoremstyle{plain}
\newtheorem{thm}{Theorem}[section]

\newtheorem{lem}[thm]{Lemma}
\newtheorem{cor}[thm]{Corollary}
\theoremstyle{definition}

\theoremstyle{remark}

\usepackage[disable,textsize=tiny]{todonotes}

\icmltitlerunning{Data-Driven Subgroup Identification for Linear Regression}

\begin{document}

\twocolumn[
\icmltitle{Data-Driven Subgroup Identification for Linear Regression}




\begin{icmlauthorlist}
\icmlauthor{Zachary Izzo}{math}
\icmlauthor{Ruishan Liu}{bds}
\icmlauthor{James Zou}{bds}
\end{icmlauthorlist}

\icmlaffiliation{math}{Department of Mathematics, Stanford University, USA}
\icmlaffiliation{bds}{Department of Biomedical Data Science, Stanford University, USA}

\icmlcorrespondingauthor{Zachary Izzo}{zizzo@stanford.edu}

\icmlkeywords{Subgroup Discovery}

\vskip 0.3in
]



\printAffiliationsAndNotice{}  

\begin{abstract}
Medical studies frequently require to extract the relationship between each covariate and the outcome with statistical confidence measures. To do this, simple parametric models are frequently used (e.g. coefficients of linear regression) but usually fitted on the whole dataset. However, it is common that the covariates may not have a uniform effect over the whole population and thus a unified simple model can miss the heterogeneous signal. For example, a linear model may be able to explain a subset of the data but fail on the rest due to the nonlinearity and heterogeneity in the data. In this paper, we propose \modelname (data-driven group discovery), a data-driven method to effectively identify subgroups in the data with a uniform linear relationship between the features and the label. \modelname outputs an interpretable region in which the linear model is expected to hold. It is simple to implement and computationally tractable for use. We show theoretically that, given a large enough sample, \modelname recovers a region where a single linear model with low variance is well-specified (if one exists), and experiments on real-world medical datasets confirm that it can discover regions where a local linear model has improved performance. Our experiments also show that \modelname can uncover subgroups with qualitatively different relationships which are missed by simply applying parametric approaches to the whole dataset.
\end{abstract}

\section{Introduction} \label{sec: intro}
In scientific and medical analyses, simple parameteric models are frequently fit to data to draw qualitative or quantitative conclusions about the relationships between different variables of interest. Typically, a single interpretable model is fit on the entire dataset, implicitly assuming that there are uniform relationships between the covariates and target variable across the whole population. In practice, the data may instead come from a heterogeneous population, where different \emph{subgroups} of the population may obey qualitatively different trends.

For example, suppose we fit a linear model with features including several patient biomarkers, as well as blood concentration of a particular drug, to predict blood pressure. After fitting the model to the whole dataset, we find that there is a statistically signficant negative coefficient on the drug concentration. We may be tempted to conclude that this drug should be administered to a general patient in order to reduce blood pressure. However, there may be a small subgroup in the data (say, patients over the age of 80) for whom the drug actually \emph{increases} blood pressure. In this case, naively fitting a single model to the entire dataset not only reduces our predictive accuracy, it also leads to adverse outcomes for this subgroup of the population.

Modern high-capacity models such as neural networks can help to avoid this problem as they represent a much richer function class. However, these models are often inherently difficult to interpret, making them unsuitable if the primary goal is to draw scientific or clinical conclusions about the data rather than simply having good predictive performance. This motivates our desire to find interpretable regions in the data where interpretable models (such as linear regression) perform well. We call this the \emph{subgroup selection} problem.

\subsection{Our Contributions}
In this work, we consider a flexible formalization of the subgroup selection problem. We propose an general algorithmic framework and a specific instantiation, \modelname (data-driven group discovery), for data-driven subgroup selection. We prove that \modelname has desirable theoretical properties, and results on synthetic and real data show the effectiveness of \modelname in practice.

\subsection{Related Work}
Subgroup identification is an important topic in biostatistics \citep{lipkovich2017subgroup}. Here, the main focus is on identifying subsets of the population with a significant beneficial treatment effect from a new drug or procedure. Common approaches include \emph{global outcome modeling}, in which the user models the patient response with and without treatment separately, then reconstructs the treatment effect from these models; \emph{global treatment modeling}, in which the user models the treatment effect directly; and \emph{local modeling}, where the user tries to identify a region with a strong positive treatment effect. Of these approaches, our method is most closely related to the local modeling approach. However, existing local modeling methods typically use tree-based greedy approaches to region selection which do not come with any guarantees \citep{lipkovich2017subgroup}.

In the knowledge discovery in databases (KDD) community, the problem of subgroup discovery has been studied more extensively; see \cite{atzmueller2015subgroup} and \cite{song2016subgroup} for surveys. In its most general form, subgroup discovery in this community refers to finding regions of the data with ``interesting'' properties, typically quantified by the use of a score function. For instance, a basic subgroup discovery method may try to find regions of the data where the mean or distribution of some target features are markedly different from the rest of the data. Later work has addressed more complicated tasks such as finding regions with exceptional regression models \citep{duivesteijn2012different} or regions in which some pre-specified ML model works well \citep{sutton2020identifying}. Subgroups in this context are often specified by a \emph{pattern}, which in the KDD literature refers to (usually pre-defined) selector variables. For instance, these could be some pre-defined thresholds on the features. Selection of the best subgroup with respect to the chosen score function then typically proceeds via either an exhaustive or greedy search over the valid patterns. The existing literature does not provide theoretical guarantees on the correctness of the selected subgroup. In contrast, we provide an efficient algorithm (not requiring exhaustive search) with provable guarantees and with data-driven (rather than pre-defined) selection criteria.

Our problem framework also has connections to list-decodable learning \citep{charikar2017list}, specifically list-decodable linear regression \citep{karmalkar2019list, raghavendra2020list}. In the list-decodable setting, we assume that an $\a$ fraction of the data come from a ``trusted'' source which we are trying to model; this would correspond to the subset of our data belonging to the region we are trying to detect. The goal is to output a small list (polynomial in $\a^{-1}$) which contains a model that will perform well on the trusted data. While an algorithm for the list-decodable linear regression problem will return a model that performs well for the ``good'' region, it does not directly solve the problem of actually finding this region itself.

Piecewise linear regression is another method for adding flexibility to linear models while preserving interpretability. Here, the assumption is that the response is a piecewise linear function of the covariates. Early works focused on the one-dimensional covariate case \citep{vieth1989piecewise}, and recently methods have been proposed for piecewise linear regression in higher dimensions \citep{siahkamari2020piecewise, diakonikolas2020piecewise}. Unlike the piecewise linear setting, we make no assumptions on the regression function outside of the ``good'' region which we are trying to detect.

Our work is also similar in spirit to previous works on conditional linear regression \citep{juba2017conditional, calderon2020conditional}. In this setting, the goal is also to find the largest possible subset of the data for which there is an accurate linear model. However, similar to many methods from the KDD literature, the subgroup identification in this case is made in terms of \emph{pre-defined} binary features, which are assumed to be provided with the data in addition to the regressor variables. While one could instantiate our problem by defining the binary inclusion variables as indicators of whether or not each regressor is above or below a certain threshold, doing so would result in exponentially many possible selection rules and will therefore be computationally intractable for our setting. One can also view our work as finding data-driven binary inclusion labels for the conditional linear regression problem.

A core element of our problem setting is in selecting a region which avoids certain ``bad'' points. Related problems have been extensively studied in the computational geometry community \citep{dobkin1988box, backer2010box, dumitrescu2013boxapprox}, but even approximate algorithms for solving related problems are not practical for high dimensions, and indeed even some seemingly simple region selection problems can be shown to be NP hard \citep{backurs2016boxhard}. We propose tractable alternatives and show that they have desirable properties both theoretically and empirically.

As we seek to learn a subset of the data on which we are willing to make predictions, our work is connected to the literature on learning with rejection \citep{cortes2016rejection} or learning to defer \citep{madras2018defer, mozannar2020defer, keswani2021defer}, in which a model is given the option not to make a prediction. These works focus primarily on classification and decide whether or not to make a prediction on individual data point via thresholding model confidence. While this implicitly defines a subgroup on which we expect the model to perform well---namely, the points for which the model does not defer---, this subgroup will typically be uninterpretable (if the model is a neural network). If logistic regression is used, the subgroup will be the complement of a slab between two parallel hyperplanes, which may be considered interpretable but is fairly inflexible in terms of the region selected. \cite{wiener2012pointwise} also considered a similar model in the regression setting, where the learner must simultaneously learn a regression function $f$ and a selection function $g$ which specifies the group of points on which to make predictions. Similar to learning with rejection, in this case, the subgroup is defined implicitly via $g$ and in general will not be interpretable. In our setting, we focus on the regression problem and on explicitly defining an interpretable region in which we will not defer.

\section{Problem Setup} \label{sec: setup}
The general subgroup selection problem can be formulated as follows. Let $\cZ = \cX \times \cY$ denote the sample space, $\cF \subseteq \cY^\cX$ denote a class of functions (e.g. linear regression models), and let $\ell: \cY \times \cY \rightarrow \R$ be a loss function measuring the performance of our model. We will always have $\cX \subseteq \R^d$ and $\cY \subseteq \R$. Our goal is to find a (interpretable, large as possible) region $R\subseteq \cX$ of the \emph{feature space} where the loss $$\argmin_{f\in \cF} \E[\ell(y, f(x)) \: | \: x \in R]$$ is small. In order to satisfy the interpretability criterion, we will consider regions $R$ which are axis-aligned boxes. This corresponds to a subgroup where each feature lies within a specified range (corresponding to the sides of the axis-aligned box). The algorithm we develop also easily allows the user to control the tradeoff between the size of the region and the loss of the selected model within the region.

For this paper, we will specify the function class $\cF$ to be linear models and the loss $\ell(y, \hat{y}) = (y-\hat{y})^2$ to be the squared loss. For our theoretical results, we will assume that there exists a ``good'' region $R^* \subseteq \cX$ where the linear model is well-specified with low conditional variance of $y|x$. In particular, we will assume that when $x\in R^*$, we have $y|x \sim \cN(x^\T \b, \s^2)$ for some coefficients $\b$. In this case, the goal will be to recover $R^*$.

\section{Algorithmic Framework} \label{sec: alg}
We introduce an algorithmic framework with three distinct phases.
\begin{itemize}
    \item \textbf{Phase 1:} Compute a rough approximation to the regression function in the good region.
    \item \textbf{Phase 2:} Using the approximate fit, define labels $\ell_i$ for each point in the training data, where $\ell_i \approx \I\{x_i \textrm{ could not reasonably belong to } R^*\}$.
    \item \textbf{Phase 3:} Find a large region which contains no rejected points.
\end{itemize}

In this work, we give specific implementations of each phase, but we note that this general framework is modular and can likely be modified to work in other settings (e.g. classification or survival analysis).

In Phase 1, we find a ``core set'' of points which should belong to the good region, then fit a model to these points. %
%
For Phase 2, we reject points by thresholding the residuals from the model found in Phase 1. For Phase 3, we remark that even if it is known which points should be included or excluded from the region, actually computing the largest region consistent with these points is NP hard, even if we restrict ourselves to axis-aligned boxes \citep{backurs2016boxhard}.

\paragraph{Phase 1:}
We denote a dataset $D = (X, Y)$ to be a collection of $n$ feature vectors (collected in $X \in \R^{n\times d}$) and corresponding labels (collected in $Y \in \R^n$). Here $\textsc{KNN}(x, k, D)$ denotes the $k$ nearest neighbors of $x$ (and their corresponding labels) in the dataset $D$,  $\textsc{OLS}(D)$ denotes the output of ordinary least squares on feature matrix $X$ and response vector $Y$, and $\textsc{MSE}(\hat{\beta}, D)$ denotes the mean squared error of linear model $\hat{\beta}$ on the data $X, Y$.

The pseudocode for selecting the core group is provided in Algorithm \ref{alg: core group}. Given a choice of core group size $k$, for each datapoint, we fit a local model to that point's $k$ nearest neighbors. We then select the group of points with the lowest training error of its local model as the core group.

\begin{algorithm}[h!]
\caption{\textsc{CoreGroup}$(k, D)$} \label{alg: core group}
\begin{algorithmic}
\INPUT Core group size $k$, dataset $D$
\STATE $\textrm{MSE}^* \gets \infty$
\FOR{$(x, y) \in D$}
    \STATE $D_\textrm{nbhd} = (X_\textrm{nbhd}, Y_\textrm{nbhd}) \gets \textsc{KNN}(x, k, D)$
    \STATE $\hat{\beta} \gets \textsc{OLS}(X_\textrm{nbhd}, Y_\textrm{nbhd})$
    \IF{$\textsc{MSE}(\hat{\b}, D_\textrm{nbhd}) < \textrm{MSE}^*$}
        \STATE $D_\textrm{core} \gets D_\textrm{nbhd}$
        \STATE $\textrm{MSE}^* \gets \textsc{MSE}(\hat{\b}, D_\textrm{nbhd})$
    \ENDIF
\ENDFOR
\OUTPUT $D_\textrm{core}$
\end{algorithmic}
\end{algorithm}

\paragraph{Phase 2:}
For our theoretical results, we use the threshold
\begin{equation} \label{eq: grow method threshold}
\rho^\textrm{grow}_{\s, n} = 2.1\s \sqrt{\log n}.
\end{equation}
Here $n$ is the size of the training set. The inclusion labels $\ell_i$ are then computed as $\ell_i = \I\{|y_i - \hat{\b}^\T x_i| \geq \rho^\textrm{grow}_{\s, n}\}$. We define the set of \emph{rejected points} $X_\textrm{rej} = \{x_i \in X \: | \: \ell_i = 1\}$. For our empirical results, the threshold will be considered as a hyperparameter and chosen using a validation set. For more detail, refer to Section~\ref{sec: experiments}.

\paragraph{Phase 3:}
Let $U \subseteq \R^d$. We define the \emph{directed infinity norm} $\|x\|_{U,\infty}$ by
$$ \|x\|_{U,\infty} = \max_{u \in U} \: x^\T u. $$
We note that for many sets $U$, $\|\cdot\|_{U,\infty}$ may not be a norm, nor even a seminorm. In what follows, $U$ will initially be defined as $U = \{\pm e_i\}_{i=1}^d$, in which case $\|\cdot\|_{U,\infty} = \|\cdot\|_\infty$ coincides with the usual infinity norm on $\R^d$. We will then gradually remove directions which are no longer relevant to consider.

The region will be described in terms of linear constraints. We will overload notation and use a set $R = \{(u_i, a_i)\}_{i=1}^m$ of constraint directions and values to denote the region $R = \{x \in B \: : \: x^\T u_i \leq a_i\}$.

The pseudocode for the growing box is provided in Algorithm \ref{alg: growing box}. When $U = \{\pm e_i\}_{i=1}^d$, Algorithm~\ref{alg: growing box} begins expanding an $\ell_\infty$ ball centered at $\bar{x}$ with each side growing at an equal rate. Whenever one of the sides runs into a rejected point, we add the corresponding linear constraint and continue growing the other sides of the box. (The directed infinity norm is what we use to measure which point will collide with the box next. For a discussion on the geometric intuition for this step, see Appendix~\ref{appendix: directed infty norm}.) This continues until all sides of the box have a support point, or there are no points left to constrain the box.

Note that the set $U$ simply specifies the normal vectors to the sides of the constraint polytope. The lengths of these vectors effectively determine the speed at which the constraint region will grow in that direction. By changing $U$, this method can select polytopes of any desired shape. Since axis-aligned boxes provide easily interpretable inclusion criteria, we use such regions for all of our experiments.

\begin{algorithm}[h!]
\caption{\textsc{GrowBox}$(\bar{x}, X_{\textrm{rej}}, U)$} \label{alg: growing box}
\begin{algorithmic}
\INPUT Starting point (center) $\bar{x}$, rejected points $X_{\textrm{rej}}$, normal vectors defining the shape of the selected region $U$
\STATE $X_\textrm{rej} \gets X_\textrm{rej} + \{-\bar{x}\}$ \COMMENT{Center the points at $\bar{x}$. $+$ denotes Minkowski sum.}
\STATE $\hat{R} \gets \emptyset$
\WHILE{$X_{\textrm{rej}} \neq \emptyset$}
    \STATE $x^* \gets \argmin_{x\in X_{\textrm{rej}}} \{ \|x\|_{U, \infty} \}$
    \STATE $a^* \gets \|x^*\|_{U, \infty}$
    \STATE $u^* \gets \argmax_{u \in U} \{ u^\T x^* \}$ \COMMENT{$u^*$ is the next support direction for the polytope}
    \STATE Add $(u^*, a^*)$ to $\hat{R}$
    \STATE Remove $u^*$ from $U$
    \STATE $X_{\textrm{rej}} \gets \{x \in X_{\textrm{rej}} \: | \: x^\T u^* < a^* \}$
\ENDWHILE
\OUTPUT $\hat{R} + \{\bar{x}\}$ \COMMENT{Undo the centering procedure from the first part of the algorithm.}
\end{algorithmic}
\end{algorithm}

Combining Phases 1-3 gives an algorithm for automatic subgroup selection. We summarize the entire pipeline in Algorithm~\ref{alg: pipeline}. Note that if the variance $\s^2$ is not known, we can replace it with a standard unbiased estimate computed on the core group.

We also remark that after $\hat{R}$ has been selected, rather than using the coefficients $\hat{\b}$ fit just to the core group, we can also choose to re-fit $\hat{\b}$ on all of the training points contained in $\hat{R}$. We use this additional step in our experiments, but it does not affect our theoretical results.

\begin{algorithm}[h!]
\caption{\textsc{DDSubgroup}$(k, U, D)$} \label{alg: pipeline}
\begin{algorithmic}
\INPUT Core group size $k$, normal vectors defining the shape of the selected region $U$, dataset $D$
\STATE
\STATE \textbf{Phase 1:} Find a core group and fit a coarse model.
\STATE $D_\textrm{core} \gets \textsc{CoreGroup}(k, D)$
\STATE $\hat{\b} \gets \textsc{OLS}(D_\textrm{core})$
\STATE

\STATE \textbf{Phase 2:} Label which points should be excluded.
\FOR{$i = 1, \ldots, n$}
    \STATE $\ell_i \gets \I\{|y_i - \hat{\b}^\T x_i| \geq \rho^{\textrm{grow}}_{\s, n}\}$
\ENDFOR
\STATE $X_\textrm{rej} \gets \{x_i \in X \: | \: \ell_i = 1\}$
\STATE

\STATE \textbf{Phase 3:} Approximate $R^*$.
\STATE $\bar{x} \gets \textsc{Mean}(X_\textrm{core})$
\STATE $\hat{R} \gets \textsc{GrowBox}(\bar{x}, X_\textrm{rej}, U)$
\OUTPUT $\hat{R}$
\end{algorithmic}
\end{algorithm}

\paragraph{Runtime}
The runtime of \modelname as described by Algorithm~\ref{alg: pipeline} is $O(kn\log n)$. We treat the dimension as a constant. After constructing a K-D tree in $O(n\log n)$ time, the $k$-nearest neighbors of a point can be found in time $O(k\log n)$. Computing the OLS fit on $k$ points in constant dimension takes $O(k)$ time, making the runtime for each step of the core group search $O(k\log n)$ since we do not need to re-compute the K-D tree for each of these steps. This step is repeated $n$ times, once for each candidate core group. The box expansion requires only $O(n)$ work once the core group has been determined, thus the overall runtime for the algorithm is $O(n\log n) + O(kn\log n) + O(n) = O(kn\log n)$. While this is only the cost of a single run of \modelname, these runs can easily be parallelized, making \modelname highly efficient even for large datasets and large hypeparameter searches.

\section{Theoretical Guarantees} \label{sec: theory}
In this section, we examine some of the theoretical properties of \modelname. All proofs are deferred to Appendix~\ref{appendix: proofs}. In what follows, ``with high probability'' means with probability approaching 1 as $n, k \rightarrow \infty$. We make the following assumptions.
\begin{enumerate}
    \item The samples $(x_i, y_i) \iid \cP$ for a probability distribution $\cP$ on $\cZ$. We let $S = \textrm{supp}(x)$ denote the support of the marginal distribution of the features. \label{assumption: iid}
    \item The features $x$ are bounded: $\|x\| \leq B$ deterministically. \label{assumption: bdd 
    x}
    \item The marginal distribution of $x$ has a density $f$ with respect to the Lebesgue measure. Furthermore, the density is bounded from above and below on the support of $x$: $0 < c_f \leq f(x) \leq C_f < \infty$ for all $x \in S$. \label{assumption: density}
    \item There is a region $R^* \subseteq S$ in which the linear model holds. That is, conditional on $x \in R^*$, $y$ is generated according to the linear model: $x\in R^* \Rightarrow y|x \sim \cN(x^\T \b^*, \s^2)$ for some fixed $\b^*$. \label{assumption: linear}
    \item The region $R^*$ is an axis-aligned box with nonempty interior, i.e. $R^* = \prod_{i=1}^d [a_i, b_i]$ for some $a_i < b_i$. \label{assumption: nonempty}
    \item Conditional on $x\not\in R^*$, $y$ is Gaussian with variance at least $\s_0^2$, where $\s_0 \geq C\s$ for some absolute constant $C$. \label{assumption: outside}
\end{enumerate}
The lower bound in Assumption~\ref{assumption: density} ensures that the samples will cover the sample space (so that we can detect $R^*$). The upper bound prevents degeneracies, e.g., if the feature distribution contains atoms with large enough mass, the KNN of certain points may contain many copies of a single point. Assumption~\ref{assumption: linear} ensures that our model is well-specified on $R^*$. Assumption~\ref{assumption: outside} ensures that $R^*$ is in fact the ``best'' region for us to select, namely, there is no other region where we can have better predictive power. This condition also ensures that the random fluctuations in $y_i$ are large enough to be detected by the test. We remark that the absolute constant $C$ is no greater than 50, but we have made no effort to optimize this constant in our analysis and it can certainly be reduced.


Our first result shows that almost all of the group selected by Algorithm~\ref{alg: core group} lies in $R^*$.
\begin{restatable}{lem}{goodcore}
\label{thm: good core whp}
The core group selected by Algorithm~\ref{alg: core group} has $X_{\mathrm{core}} \setminus R^* = o(k)$ with high probability.
\end{restatable}

The next result states that we will not erroneously reject any points that actually belong to $R^*$.
\begin{restatable}{lem}{nobadrej} \label{thm: no bad rej}
Let $X_{\mathrm{core}}$ be the core group selected by Algorithm~\ref{alg: core group} and let $\hat{\b}$ be the OLS estimator fit to $X_{\mathrm{core}}$. Let $X_\textrm{rej}$ be the set of rejected points defined by the thresholding procedure in Phase 2. With high probability, none of the points in $X_\textrm{rej}$ belong to $R^*$.
\end{restatable}

Combining Lemmas~\ref{thm: good core whp} and \ref{thm: no bad rej}, we show that \modelname precisely recovers $R^*$ given sufficient data.
\begin{restatable}{thm}{main} \label{thm: main}
As $n\rightarrow\infty$, there exist positive scalars $\{s^\pm_j\}_{j=1}^d$ and a constant $c > 0$ such that if $U = \{s^+_j e_j, -s^-_j e_j\}_{j=1}^d$ and $k = \Omega(n)$ with $k \leq cn$, Algorithm~\ref{alg: pipeline} returns $\hat{R}$ with $R^* \subseteq \hat{R}$ with high probability. Furthermore, $\mathrm{vol}(\hat{R} \setminus R^*) \rightarrow 0$.
\end{restatable}
We remark that the scalars $s^\pm_j$ can depend on the dataset and the constant $c$ may depend on $R^*$ and the other parameters in Assumptions~\ref{assumption: iid}-\ref{assumption: outside}.

As an immediate corollary to Theorem~\ref{thm: main}, we see that under slightly modified assumptions, \modelname can be used to find multiple subgroups in the data by iteratively applying Algorithm~\ref{alg: pipeline}.
\begin{cor} \label{thm: multi}
Suppose that Assumptions~\ref{assumption: iid}-\ref{assumption: outside} hold, but instead of a single region $R^*$, there are multiple disjoint regions $R_g$, $g=1,\ldots,G$ where for $x \in R_g$, $y|x \sim \cN(\b_g^\T x, \s_g^2)$. Furthermore, assume that Assumption~\ref{assumption: outside} holds with $\s_0 > C\s_g$ whenever the $x_i \not\in \bigcup_{g=1}^G R_g$. Let $\hat{R}_g$, $g=1,\ldots,G$ be the outputs after running Algorithm~\ref{alg: pipeline} $G$ times, removing the training points which are contained in $\hat{R}_g$ after the $g$-th run. Then under the same conditions as in Theorem~\ref{thm: main}, we have $R_g \subseteq \hat{R}_g$ and $\vol(\hat{R}_g \setminus R_g) \rightarrow 0$ for all $g$.
\end{cor}

\section{Experiments} \label{sec: experiments}
In this section, we evaluate the performance of \modelname on both synthetic and real-world medical datasets.

\paragraph{Methods for Comparision} We compare \modelname with several other baselines.
\begin{enumerate}
\item Standard linear regression, i.e., a linear model fit to the whole dataset. It is equivalent to the situation where the selected region includes all of the data and it is the method employed by the original medical studies on the real-world datasets we consider.
\item An unsupervised clustering method. Here we use $k$-means clustering and identify the cluster with the smallest MSE as the most coherent subgroup. We use the bounding box defined by the selected subgroup as the interpretable inclusion criteria.
\item Linear model trees. These are decision trees with a linear regression model in each leaf \citep{wang1996induction, potts2005incremental}.
Though LMT is not designed for subgroup identification, we can still use its decision path as a way to select cohorts. In order to identify the most coherent subgroup, we pick the leaf of the LMT with the smallest MSE. 
\end{enumerate}

\paragraph{Experiment Setup} 
For the real-world datasets, we randomly split them into training, test and validation sets, with ratio 50\%, 30\% and 20\%. In each experiment, we fit the models on the training set with a grid search over hyperparameters and select the region with lowest validation MSE. We then refit the linear model on the training points in the selected region and evaluate its performance on the test set.
For \modelname,
we used a more general form of the threshold $\rho_{\g_1, \g_2}(x_i) = \s\g_1 \|x_i\| + \s\g_2$ and tuned $\g_1$ and $\g_2$ as additional hyperparameters. Specifically, the algorithm works well by simply setting $\g_2 = 0$ and tuning $\g_1 \in \{2^{-4}, 2^{-3}, \ldots, 2^5\}$. We also set the size $k$ of the core group equal to $p$ times the size of the training set, where $p$ was selected from within $\{0.01, 0.05, 0.1, 0.15, 0.2\}$. We also tried two different ``speed'' settings for Algorithm~\ref{alg: growing box}: the sides of the box either grow all at the same rate, or each side grows at a rate proportional to the length of the bounding box $B$ in that dimension.
For $k$-means clustering, the number of clusters is a critical parameter and is scanned from 2 to twice the dimension of the data for the best performance.
For LMT, the tree depth is an important parameter and is scanned from 1 to the dimension of the data on the validation set for the best performance.

\subsection{Demonstration on Synthetic Data}
To visualize our method and test its performance in a well-specified setting, we construct a synthetic dataset where the desired region to be selected is known. Let $B \subseteq \R^d$ be the feature space, and let $R^* \subseteq B$ be the ``true'' region that we wish to recover. The data are generated as follows. We first sample the features $x \sim \mathrm{Unif}(B)$. If $x \in R^*$, set $y = \b^\T x + \e_\mathrm{in}$. Else if $x \not\in R^*$, set $y = \e_\mathrm{out}$. Here $\b \neq 0 \in \R^d$ are the fixed true model weights for the region $R^*$. The error terms $\ein$ and $\eout$ follow $\ein \sim \cN(0, \sin^2)$ and $\eout \sim \cN(0, \sout^2)$ with $\sin < \sout$. We set the dimension $d=3$ so that the selected region can be easily visualized. (The third dimension just allows us to incorporate a bias term, so we will only visualize two dimensions.) We define the bounding box for the features $B = [-1, 1]^2 \times \{1\}$ and the true region $R = [-1/3, 1/3] \times [-2/3, 2/3] \times \{1\}$, and we generate $n=1000$ data points.

Figure~\ref{fig: synthetic region} shows the results of running Algorithm~\ref{alg: pipeline} on this synthetic data. The gray shaded region is $R^*$. The red ``x'' (resp. blue ``o'') markers denote points that were rejected (resp. not rejected) by the threshold \eqref{eq: grow method threshold}, and the green rectangle shows the boundary of $\hat{R}$ returned by \modelname. There is a nearly perfect overlap between $R^*$ and $\hat{R}$, meaning \modelname is able to precisely recover the true region.
In contrast, the green rectangles in Figure~\ref{fig: synthetic region clustering} and Figure~\ref{fig: synthetic region LMT} shows the region selected by $k$-means clustering and LMT. The $k$-means clustering method erroneously excludes points within the correct subgroup, while LMT tends to select points outside of $R^*$.

Figure~\ref{fig: synthetic robustness} shows the robustness of \modelname to a misspecified core group. We replace the output of Algorithm~\ref{alg: core group} with a manually supplied set of points. We start by providing a core group whose center coincides with that of $R^*$. The $x$-axis of the plot denotes the offset of this initial core group: at position $x$ on the plot, the center of the core group has been shifted by $(x,x)$. Because we grow the sides of $\hat{R}$ at the same speed, it becomes harder to recover the full $R^*$ when the center of the core group is closer to the edge of $R^*$ (larger $x$ value on the plot). We plot three quantitites:
\begin{itemize}
    \item Precision $= \vol(\hat{R} \cap R^*) / \vol(\hat{R})$,
    \item Recall $= \vol(\hat{R} \cap R^*) / \vol(R^*)$, and
    \item F1 score.
\end{itemize}
The vertical dashed black line denotes the point at which the core group starts to include points which do not belong to $R^*$. The vertical dashed red line denotes the point at which the center of the core group (and thus the base point from which we grow $\hat{R}$) lies outside of $R^*$. We see that \modelname is quite robust to the location of the core group within $R^*$. However, once ``bad'' points are included in the core group, the performance (in particular the recall) begins to drop sharply. The precision is more robust to core group misspecification, remaining well above the baseline of $0.22$ (which is equivalent to selecting the whole region) even when the core group is more than $50\%$ misspecified.


\newcommand{\subfigwidth}{.24\textwidth}
\begin{figure*}[t]
\centering
\begin{subfigure}[t]{\subfigwidth}
\centering
  \includegraphics[width=1\linewidth]{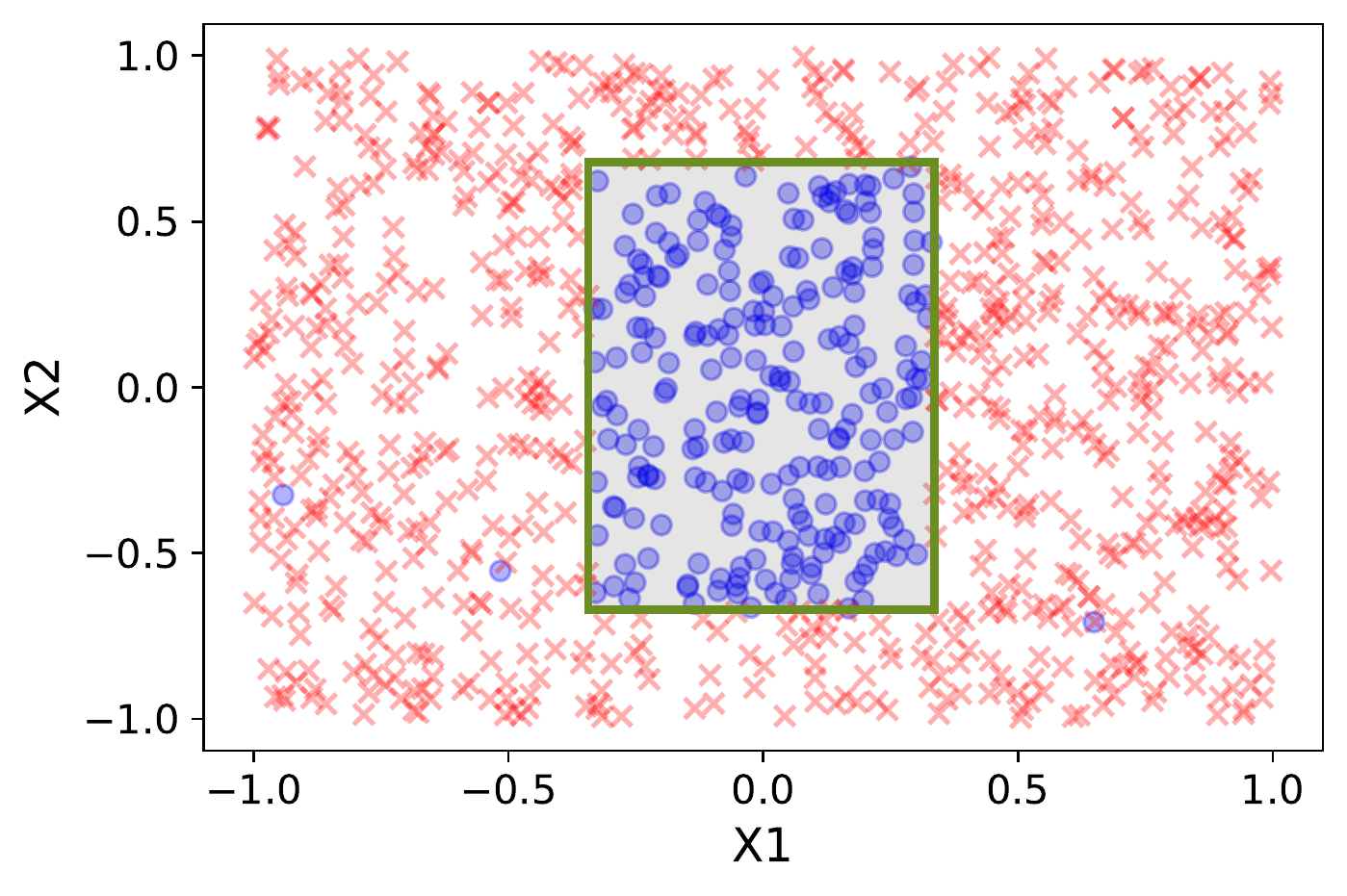}
  \caption{Region selected by \modelname.}
  \label{fig: synthetic region}
\end{subfigure} \hfill
\begin{subfigure}[t]{\subfigwidth}
\centering
  \includegraphics[width=1\linewidth]{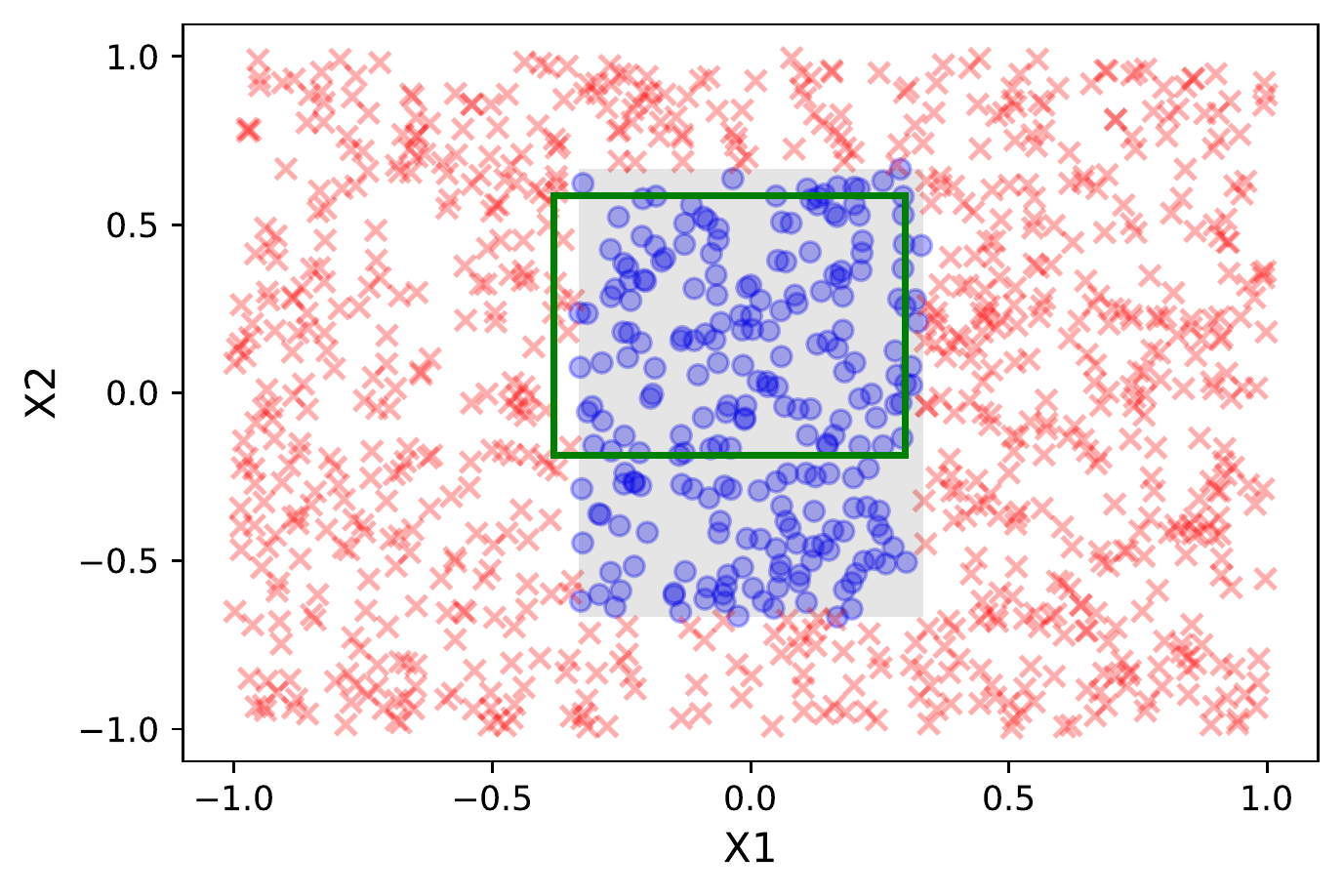}
  \caption{Region selected by clustering.}
  \label{fig: synthetic region clustering}
\end{subfigure} \hfill
\begin{subfigure}[t]{\subfigwidth}
\centering
  \includegraphics[width=1\linewidth]{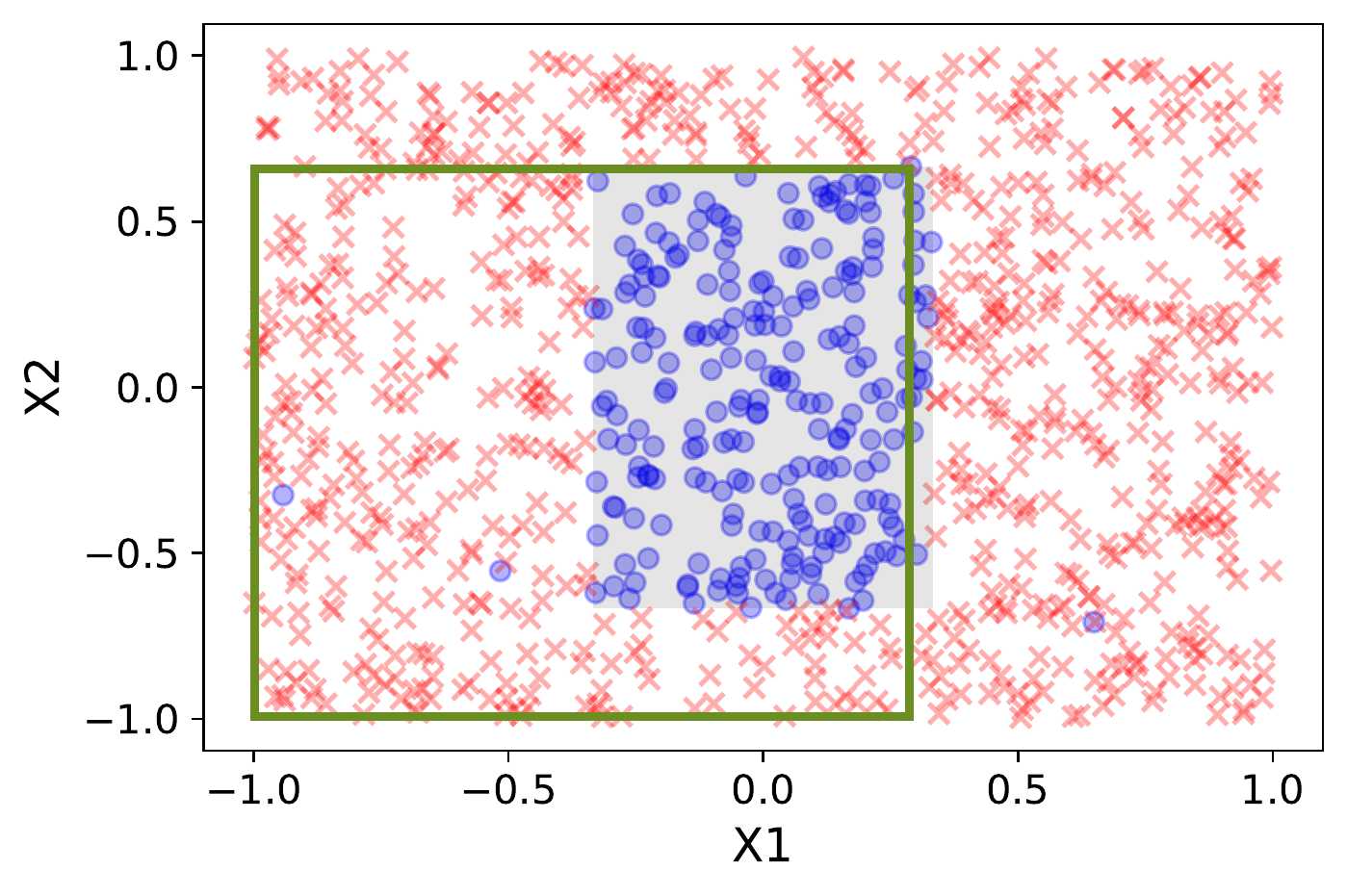}
  \caption{Region selected by LMT.}
  \label{fig: synthetic region LMT}
\end{subfigure} \hfill
\begin{subfigure}[t]{\subfigwidth}
  \centering
  \makebox[\textwidth][c]{\includegraphics[width=0.97\linewidth]{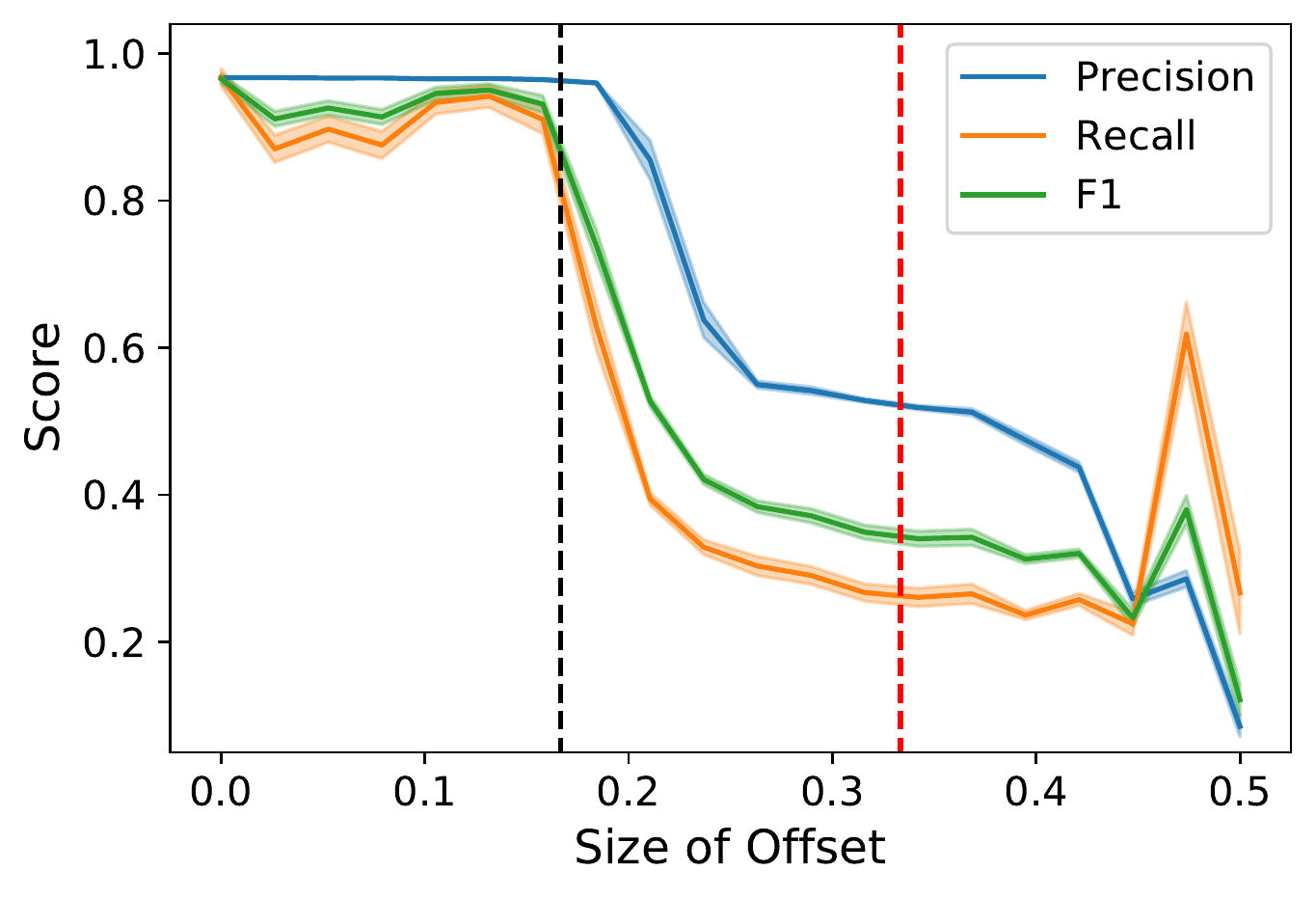}}
  \caption{Robustness of \modelname.}
  \label{fig: synthetic robustness}
\end{subfigure} 
\caption{Demonstration on synthetic dataset. (a-c) The region selected by (a) \modelname, (b) $k$-means clustering and (c) linear model tree. The grey shaded area denotes the correct subgroup and the green box corresponds to the learned boundary. For $k$-means clustering, the number of clusters is searched from 2 to 10, and the bounding box for the cluster with smallest MSE is reported in (b). The depth of LMT is searched from 1 to 10, and the best performance is reported in (c).
(d) Robustness of \modelname to core group misspecification. The shaded region shows standard error of the mean over 50 trials. The black dashed line denotes the point at which ``bad" points are included in the core region. The red dashed line denotes the point at which the center of the supplied core set is outside of $R$. The $y$-axis records precision, recall, and F1 score (higher is better).
}
\end{figure*}

Table~\ref{table: sample size} shows the performance of different methods (measured by F1 score) vs. sample size. \modelname has a statistically significant performance improvement compared to the other methods for all sample sizes. LMT eventually identifies most of the correct region, but it is much less sample efficient than \modelname. Increasing the sample size does not appear to help for the clustering method.
\setlength{\tabcolsep}{4.6pt}
\bgroup
\def\arraystretch{1}
\begin{table*}[ht!]
\caption{Performance on single subgroup identification for $k$-means clustering, linear model tree, and \modelname on synthetic datasets of varying sizes. We report the average F1 score plus or minus the standard error of the mean over 20 trials. \modelname outperforms the comparison methods for all sample sizes and finds accurate results even with few samples.
}
\label{table: sample size}
\centering
\begin{tabular}{c|ccccccc}
\toprule  
$n$ & 200 & 400 & 800 & 1600 & 3200 & 6400 & 12800 \\  \midrule
\modelname & \bf{0.73 ± 0.03} & \bf{0.68 ± 0.07} & \bf{0.93 ± 0.02} & \bf{0.98 ± 0.00} & \bf{0.99 ± 0.00} & \bf{0.99 ± 0.00} & \bf{1.00 ± 0.00} \\  \midrule
LMT        & 0.23 ± 0.09 & 0.32 ± 0.10 & 0.19 ± 0.09 & 0.28 ± 0.10 & 0.48 ± 0.11 & 0.92 ± 0.05 & 0.93 ± 0.05 \\  \midrule
Clustering & 0.07 ± 0.04 & 0.18 ± 0.07 & 0.02 ± 0.01 & 0.12 ± 0.05 & 0.12 ± 0.06 & 0.14 ± 0.07 & 0.11 ± 0.06 \\  
\bottomrule
\end{tabular}
\end{table*}
\egroup

\subsection{Evaluation on Real-World Datasets}
We further evaluate our method on five real-world medical related datasets, where linear coefficients were used for interpretation in their original publications.
\begin{enumerate}[leftmargin=0.7cm]
    \item Brazil Health Dataset \citep{cavalcante2018did} is from a longitudinal ecological study for 645 municipalities in the state of São Paulo, Brazil. The study uses a linear model to identify key features for hospitalization of heart failure (HF) and strokes.
    \item China Glucose Dataset \citep{wang2017fasting} consists of 5,726 female (F) and 5,457 male (M) Chinese individuals with normal glucose tolerance. The study uses linear model to describe the relationship between fasting plasma glucose and serum uric acid levels (SUA).
    \item China HIV Dataset \citep{zhang2016stigma} consists of 2,987 participants living with HIV from Guangxi province, China. The study uses linear regression to study how routes of HIV infection affect the HIV  internalized stigma scale, adjusted by patients' characteristics.
    \item Dutch Drinking dataset \citep{boelema2015adolescent} consists of the individual life survey data of alcohol use among 2,230 Dutch adolescents. The study uses linear regression to analyze how drinking affects adolescents' inhibition (inh), working memory (wm) and shift attention (sha).
    \item Korea Grip Dataset \citep{wen2017association} is for the Dong-gu study of 2,251 Korean adults with osteoarthritis (OA). The study uses linear regression to explore the associations between grip strength and individual radiographic feature scores of OA.
\end{enumerate}

\paragraph{Performance Evaluation for a Single Group}  
We first examine the case in which we try to find a single subgroup of the data. Table~\ref{table:real} shows the mean test MSE and fraction of test points included in the selected region (both $\pm$ the standard error of the mean) averaged over 10 random train/validation/test splits. \modelname correctly identifies a subgroup on which the linear model has low test error and consistently outperforms the baseline methods on all five real-world medical datasets. Across all of the datasets, it most frequently has the lowest test MSE, and \emph{never} has a test MSE which was statistically significantly worse than any other method. We demonstrate that there exist subgroups within the real-world population where a linear model is a good proxy and should be used to enhance interpretability. Our current method focuses on finding the most coherent region within the dataset, thus it always identifies  small subgroups with the strongest signal. If a larger subgroup is desired, one may enforce this by selecting the best region which includes e.g. at least a certain fraction of the validation set. In our case, we required that at least 5\% of validation was selected. We also remark that \modelname is computationally efficient in practice. The average runtime for Algorithm~\ref{alg: pipeline} across one run of each dataset was 1.98 seconds on an AMD 7502 CPU, and no individual dataset took longer than 10 seconds.

\paragraph{Performance Evaluation for Multiple Groups}
Next, we examine the performance of the competing methods when selecting multiple subgroups in the data; in this case, we select three subgroups. We modified the \modelname procedure according to Corollary~\ref{thm: multi}. For the other methods, we performed a similar iterative procedure, repeatedly removing the selected subgroups from the training data. Table~\ref{table:multi} shows the same statistics as reported in Table~\ref{table:real}, but averaged across the three selected subgroups (as well as the random train/validation/test splits). When selecting multiple subgroups, \modelname maintains its advantage over the other methods. As in the single group case, it usually has the lowest test MSE and is never statistically significantly worse than any other method.

\paragraph{Case Study} 
Here we use the China HIV Dataset to illustrate how \modelname can enhance our understanding of the data. The original study analyzes how different HIV infection routes affect the internalized stigma by fitting a multivariate linear regression model with confounders \citep{zhang2016stigma}. In their main results, the blood transfusion route is found to 
have positive effect on internalized stigma (coefficient $\beta$ larger than zero), but in low confidence with a large $p$ value. 
In our analysis, we observed similar behavior: after data standardization, the linear model fit to the whole dataset predicts blood transfusion route to have a positive effect on internalized stigma with $\beta$ of 0.12, but low confidence level with a $p$ value of 0.67. On the other hand, \modelname identifies a subgroup of 21\% of the participants where blood transfusion route has the opposite effect on stigma ($\beta=-1.71$) with a strong signal ($p= 0.006$).
The selected subgroup consists of younger participants with
lower self-esteem, lower anxiety level, and less social support. 
The result indicates that while blood transfusion route seems not to associate with internalized stigma in the general population living with HIV, it is coherently associated with lower stigma in a certain subpopulation.
This seems plausible, as the other infection routes include sex with stable partners, sex with casual partners, sex with commercial partners, and injecting drug use. Younger participants may have stronger feelings of shame associated with these activities than older participants.
In general, interpretation of the learned selection rules could be of great interest in real applications.

Before concluding the section, we remark that \modelname offers flexibility in choosing between the size of the selected subgroups and the MSE of the linear model on these subgroups. In these experiments, we required that the selected subgroups contain at least $5\%$ of the validation set in order to be considered. This threshold can easily be modified, and in general a higher threshold will encourage the selection of larger regions at the expense of a higher MSE. See Appendix~\ref{appendix: mse vs size} for more details and experiments regarding this tradeoff.

\setlength{\tabcolsep}{4.6pt}
\bgroup
\def\arraystretch{1}
\begin{table*}[ht!]
\caption{Performance on single subgroup identification for baseline (linear regression model on the whole data), $k$-means clustering, linear model tree, and \modelname on the real-world datasets. Here $d$ denotes the dimension of the features, and subgroup size denotes the fraction of the data included in the selected subgroup. We report the average results ($\pm$ the standard error of the mean) for 10 runs of different random splits.
}
\label{table:real}
\centering
\begin{tabular}{lcc|cccc|ccc}
\toprule
\multirow{2}{*}{Dataset} & \multirow{2}{*}{Task} &
\multirow{2}{*}{$d$} & \multicolumn{4}{c|}{Test MSE} &  \multicolumn{3}{c}{Subgroup Size}  \\  
& &  &  Baseline & Clustering & LMT  & \modelname & Clustering & LMT & \modelname \\  \midrule
 \multirow{2}{*}{\begin{tabular}[c]{@{}c@{}} Brazil \\ Health  \end{tabular}} & HF & 6 & 0.80 ± 0.06 & 0.33 ± 0.03 & 0.21 ± 0.02  &   \bf{0.04 ± 0.00} &  18\% ± 2\% & 13\% ± 1\% &   6\% ± 0\%  \\ 
 & stroke & 6 & 1.14 ± 0.22 &  0.21 ± 0.01 & 0.16 ± 0.01 &  \bf{0.06 ± 0.00} &  20\% ± 2\% & 14\% ± 1\% &   6\% ± 0\% \\  \midrule
 \multirow{2}{*}{\begin{tabular}[c]{@{}c@{}} China \\ Glucose \end{tabular}} & SUA-F & 11 & 0.83 ± 0.02 &  0.69 ± 0.03 & 0.73 ± 0.04 &   0.69 ± 0.06 &  27\% ± 2\% & 24\% ± 6\% &  21\% ± 3\% \\ 
 & SUA-M & 11 & 0.94 ± 0.01 & 0.89 ± 0.04 & \bf{0.80 ± 0.0}2 &  0.81 ± 0.04 &  21\% ± 5\% & 15\% ± 1\% &   8\% ± 1\%  \\  \midrule
 \multicolumn{2}{l}{China HIV \ stigma} & 27 & 0.84 ± 0.01 &  0.86 ± 0.08 & 0.83 ± 0.08 &  \bf{0.69 ± 0.04} &  6\% ± 1\% & 18\% ± 4\% &  21\% ± 3\% \\ \midrule
\multirow{3}{*}{\begin{tabular}[c]{@{}c@{}} Dutch \\ Drinking  \end{tabular}} & inh & 16 & 0.64 ± 0.01 & 0.56 ± 0.02 & 0.51 ± 0.03 & \bf{0.50 ± 0.02} &  11\% ± 1\% & 24\% ± 5\% &  11\% ± 2\%   \\ 
 & wm & 16 & 0.71 ± 0.01 & 0.61 ± 0.02 &  \bf{0.56 ± 0.02} &  0.57 ± 0.02 &  11\% ± 1\% & 18\% ± 3\% &   9\% ± 1\% \\  
 & sha & 16 & 0.64 ± 0.01 & 0.49 ± 0.02 & 0.47 ± 0.02 &  \bf{0.42 ± 0.02} &  14\% ± 2\% & 18\% ± 4\% &  10\% ± 1\% \\  \midrule
 \multicolumn{2}{l}{Korea Grip \ strength}  & 11 & 0.71 ± 0.02 & 0.84 ± 0.13 & 0.92 ± 0.10 & \bf{0.69 ± 0.04} & 7\% ± 1\% & 33\% ± 6\% &  20\% ± 3\% \\  
\bottomrule
\end{tabular}
\end{table*}
\egroup

\bgroup
\def\arraystretch{1}
\begin{table*}[ht!]
\caption{Performance on multiple subgroups identification for baseline (linear regression model on the whole data), $k$-means clustering, linear model tree, and \modelname on the real-world datasets. Here we select \textbf{three} subgroups (rather than a single subgroup as in Table~\ref{table:real}) and report the average results for the selected groups. Here $d$ denotes the dimension of the features, and subgroup size denotes the fraction of the data included in the selected subgroups. We report the average results for 10 runs of different random splits ($\pm$ the standard error of the mean).
}
\label{table:multi}
\centering
\begin{tabular}{lcc|cccc|ccc}
\toprule
\multirow{2}{*}{Dataset} & \multirow{2}{*}{Task} &
\multirow{2}{*}{$d$} & \multicolumn{4}{c|}{Test MSE} &  \multicolumn{3}{c}{Subgroup Size}  \\  
& &  &  Baseline & Clustering & LMT  & \modelname & Clustering & LMT & \modelname \\  \midrule
 \multirow{2}{*}{\begin{tabular}[c]{@{}c@{}} Brazil \\ Health  \end{tabular}} & HF & 6 & 0.80 ± 0.06 &  0.42 ± 0.02 & 0.35 ± 0.01  & \bf{0.21 ± 0.03} &  19\% ± 1\% & 14\% ± 0\% & 7\% ± 1\% \\ 
 & stroke & 6 & 1.14 ± 0.22 &  0.27 ± 0.00 & 0.26 ± 0.01  & \bf{0.15 ± 0.01} &  23\% ± 2\% & 15\% ± 1\% & 6\% ± 1\% \\  \midrule
 \multirow{2}{*}{\begin{tabular}[c]{@{}c@{}} China \\ Glucose \end{tabular}} & SUA-F  & 11 & 0.83 ± 0.02  &  0.75 ± 0.06 & 0.82 ± 0.03  & \bf{0.72 ± 0.03} &  29\% ± 3\% & 16\% ± 1\% & 14\% ± 2\% \\ 
 & SUA-M & 11 & 0.94 ± 0.01 &  0.92 ± 0.02 & 0.88 ± 0.02  & 0.88 ± 0.03 &  15\% ± 1\% & 16\% ± 1\% & 14\% ± 4\% \\  \midrule
 \multicolumn{2}{l}{China HIV \ stigma} & 27 & 0.84 ± 0.01  & 0.96 ± 0.04 & 0.91 ± 0.04  & \bf{0.80 ± 0.04} &  38\% ± 6\% & 16\% ± 2\% & 20\% ± 2\% \\ \midrule
\multirow{3}{*}{\begin{tabular}[c]{@{}c@{}} Dutch \\ Drinking  \end{tabular}} & inh & 16 & 0.64 ± 0.01  & 0.56 ± 0.02 & 0.55 ± 0.01 & \bf{0.49 ± 0.02} &  12\% ± 1\% & 14\% ± 1\% & 10\% ± 1\%  \\ 
 & wm & 16 & 0.71 ± 0.01 & 0.64 ± 0.01 & \bf{0.58 ± 0.01} & 0.59 ± 0.01 &  13\% ± 1\% & 13\% ± 0\% & 13\% ± 3\% \\  
 & sha & 16 & 0.64 ± 0.01 & 0.52 ± 0.01 & 0.51 ± 0.01 & \bf{0.47 ± 0.01} &  12\% ± 1\% & 14\% ± 1\% & 11\% ± 1\% \\  \midrule
 \multicolumn{2}{l}{Korea Grip \ strength}  & 11& 0.71 ± 0.02 & 0.99 ± 0.17 & 0.86 ± 0.05 & \bf{0.70 ± 0.07} &  10\% ± 3\% & 23\% ± 2\% & 23\% ± 4\% \\  
\bottomrule
\end{tabular}
\end{table*}
\egroup

\section{Discussion}
In this paper, we considered a flexible formalization of the cohort selection problem. We proposed a general algorithmic framework and a specific instantiation, \modelname, for solving the problem, and we proved that \modelname recovers the correct subgroup given sufficient data. Experiments on both synthetic and real data verify our theory and show the practical usefulness of \modelname.

\subsection{Limitations \& Future Work}
While the assumption that there is a region in which the linear model holds exactly may seem strong at first glance, if the true regression function for the data is differentiable, then a linear model will always hold locally. Thus, if it is acceptable to select a small group, \modelname can still succeed in nonlinear cases. However, if the true regression function is highly oscillatory, these locally linear regions may be very small, and a large amount of data will be required to find them. Another limitation may arise in situations where there is not a unique ``best'' region (or collection of best regions), i.e., when $\var(y|x)$ is roughly the same across the whole data space. In such cases, the regions discovered by \modelname may be unstable across different random splits of the data, as there is not a strong reason for \modelname to prefer one region of the data over another.

There are a number of important open questions which remain to be addressed. If a hyperparameter search is used with \modelname to train a linear model (as we did with our real data experiments), further analysis is needed to give meaningful (but valid) $p$-values for the resulting model coefficients. For any extensive hyperparameter search, a naive Bonferroni correction is likely to be too conservative. Another important question is how to extend our framework to classification and survival analysis data.

\subsection{Societal Impact}
In particular if this method is used for medical applications, safety concerns must always be paramount. Even if we control some notion of the false discovery rate, it is conceivable that the method will discover a region with a favorable relationship between the covariates and labels that holds only by chance, and if such a region is used to make clinical decisions, it could lead to adverse outcomes for patients. Thus biological plausibility and medical best practices must always be kept in mind when applying \modelname.

\section*{Acknowledgements}
We thank all of the anonymous reviewers for their helpful comments and feedback.
J.Z. is supported by NSF CAREER 1942926 and grants from the Silicon Valley Foundation and the Chan–Zuckerberg Initiative.
Z.I. is supported by a Stanford Interdisciplinary Graduate Fellowship.

\newpage
\bibliographystyle{plainnat}
\bibliography{aistats/paper}

\begin{thebibliography}{33}
\providecommand{\natexlab}[1]{#1}
\providecommand{\url}[1]{\texttt{#1}}
\expandafter\ifx\csname urlstyle\endcsname\relax
  \providecommand{\doi}[1]{doi: #1}\else
  \providecommand{\doi}{doi: \begingroup \urlstyle{rm}\Url}\fi

\bibitem[Atzmueller(2015)]{atzmueller2015subgroup}
Martin Atzmueller.
\newblock Subgroup discovery.
\newblock \emph{Wiley Interdisciplinary Reviews: Data Mining and Knowledge
  Discovery}, 5\penalty0 (1):\penalty0 35--49, 2015.

\bibitem[Backer and Keil(2010)]{backer2010box}
Jonathan Backer and J~Mark Keil.
\newblock The mono-and bichromatic empty rectangle and square problems in all
  dimensions.
\newblock In \emph{Latin American Symposium on Theoretical Informatics}, pages
  14--25. Springer, 2010.

\bibitem[Backurs et~al.(2016)Backurs, Dikkala, and Tzamos]{backurs2016boxhard}
Arturs Backurs, Nishanth Dikkala, and Christos Tzamos.
\newblock Tight hardness results for maximum weight rectangles.
\newblock \emph{arXiv preprint arXiv:1602.05837}, 2016.

\bibitem[Boelema et~al.(2015)Boelema, Harakeh, Van~Zandvoort, Reijneveld,
  Verhulst, Ormel, and Vollebergh]{boelema2015adolescent}
Sarai~R Boelema, Zeena Harakeh, Martine~JE Van~Zandvoort, Sijmen~A Reijneveld,
  Frank~C Verhulst, Johan Ormel, and Wilma~AM Vollebergh.
\newblock Adolescent heavy drinking does not affect maturation of basic
  executive functioning: longitudinal findings from the trails study.
\newblock \emph{PloS one}, 10\penalty0 (10):\penalty0 e0139186, 2015.

\bibitem[Calderon et~al.(2020)Calderon, Juba, Li, Li, and
  Ruan]{calderon2020conditional}
Diego Calderon, Brendan Juba, Sirui Li, Zongyi Li, and Lisa Ruan.
\newblock Conditional linear regression.
\newblock In \emph{International Conference on Artificial Intelligence and
  Statistics}, pages 2164--2173. PMLR, 2020.

\bibitem[Cavalcante et~al.(2018)Cavalcante, Brizon, Probst, Meneghim, Pereira,
  and Ambrosano]{cavalcante2018did}
Denise de F{\'a}tima~Barros Cavalcante, Val{\'e}ria Silva~C{\^a}ndido Brizon,
  Livia~Fernandes Probst, Marcelo de~Castro Meneghim, Antonio~Carlos Pereira,
  and Gl{\'a}ucia Maria~Bovi Ambrosano.
\newblock Did the family health strategy have an impact on indicators of
  hospitalizations for stroke and heart failure? longitudinal study in brazil:
  1998-2013.
\newblock \emph{PLoS One}, 13\penalty0 (6):\penalty0 e0198428, 2018.

\bibitem[Charikar et~al.(2017)Charikar, Steinhardt, and
  Valiant]{charikar2017list}
Moses Charikar, Jacob Steinhardt, and Gregory Valiant.
\newblock Learning from untrusted data.
\newblock In \emph{Proceedings of the 49th Annual ACM SIGACT Symposium on
  Theory of Computing}, pages 47--60, 2017.

\bibitem[Chatterjee(2014)]{chatterjee2014superconcentration}
Sourav Chatterjee.
\newblock \emph{Superconcentration and related topics}, volume~15.
\newblock Springer, 2014.

\bibitem[Cortes et~al.(2016)Cortes, DeSalvo, and Mohri]{cortes2016rejection}
Corinna Cortes, Giulia DeSalvo, and Mehryar Mohri.
\newblock Learning with rejection.
\newblock In \emph{International Conference on Algorithmic Learning Theory},
  pages 67--82. Springer, 2016.

\bibitem[Diakonikolas et~al.(2020)Diakonikolas, Li, and
  Voloshinov]{diakonikolas2020piecewise}
Ilias Diakonikolas, Jerry Li, and Anastasia Voloshinov.
\newblock Efficient algorithms for multidimensional segmented regression.
\newblock \emph{arXiv preprint arXiv:2003.11086}, 2020.

\bibitem[Dobkin et~al.(1988)Dobkin, Edelsbrunner, and Overmars]{dobkin1988box}
David~P Dobkin, Herbert Edelsbrunner, and Mark~H Overmars.
\newblock Searching for empty convex polygons.
\newblock In \emph{Proceedings of the fourth annual symposium on Computational
  geometry}, pages 224--228, 1988.

\bibitem[Duivesteijn et~al.(2012)Duivesteijn, Feelders, and
  Knobbe]{duivesteijn2012different}
Wouter Duivesteijn, Ad~Feelders, and Arno Knobbe.
\newblock Different slopes for different folks: mining for exceptional
  regression models with cook's distance.
\newblock In \emph{Proceedings of the 18th ACM SIGKDD international conference
  on Knowledge discovery and data mining}, pages 868--876, 2012.

\bibitem[Dumitrescu and Jiang(2013)]{dumitrescu2013boxapprox}
Adrian Dumitrescu and Minghui Jiang.
\newblock On the largest empty axis-parallel box amidst n points.
\newblock \emph{Algorithmica}, 66\penalty0 (2):\penalty0 225--248, 2013.

\bibitem[Izzo et~al.(2022)Izzo, Zou, and Ying]{izzo2022stateful}
Zachary Izzo, James Zou, and Lexing Ying.
\newblock How to learn when data gradually reacts to your model.
\newblock In \emph{International Conference on Artificial Intelligence and
  Statistics}, pages 3998--4035. PMLR, 2022.

\bibitem[Juba(2017)]{juba2017conditional}
Brendan Juba.
\newblock {Conditional Sparse Linear Regression}.
\newblock In Christos~H. Papadimitriou, editor, \emph{8th Innovations in
  Theoretical Computer Science Conference (ITCS 2017)}, volume~67 of
  \emph{Leibniz International Proceedings in Informatics (LIPIcs)}, pages
  45:1--45:14, Dagstuhl, Germany, 2017. Schloss Dagstuhl--Leibniz-Zentrum fuer
  Informatik.
\newblock ISBN 978-3-95977-029-3.
\newblock \doi{10.4230/LIPIcs.ITCS.2017.45}.
\newblock URL \url{http://drops.dagstuhl.de/opus/volltexte/2017/8151}.

\bibitem[Karmalkar et~al.(2019)Karmalkar, Klivans, and
  Kothari]{karmalkar2019list}
Sushrut Karmalkar, Adam Klivans, and Pravesh Kothari.
\newblock List-decodable linear regression.
\newblock \emph{Advances in neural information processing systems}, 32, 2019.

\bibitem[Keswani et~al.(2021)Keswani, Lease, and Kenthapadi]{keswani2021defer}
Vijay Keswani, Matthew Lease, and Krishnaram Kenthapadi.
\newblock Towards unbiased and accurate deferral to multiple experts.
\newblock In \emph{Proceedings of the 2021 AAAI/ACM Conference on AI, Ethics,
  and Society}, pages 154--165, 2021.

\bibitem[Lipkovich et~al.(2017)Lipkovich, Dmitrienko, and
  B~D'Agostino~Sr]{lipkovich2017subgroup}
Ilya Lipkovich, Alex Dmitrienko, and Ralph B~D'Agostino~Sr.
\newblock Tutorial in biostatistics: data-driven subgroup identification and
  analysis in clinical trials.
\newblock \emph{Statistics in medicine}, 36\penalty0 (1):\penalty0 136--196,
  2017.

\bibitem[Madras et~al.(2018)Madras, Pitassi, and Zemel]{madras2018defer}
David Madras, Toni Pitassi, and Richard Zemel.
\newblock Predict responsibly: improving fairness and accuracy by learning to
  defer.
\newblock \emph{Advances in Neural Information Processing Systems}, 31, 2018.

\bibitem[Mozannar and Sontag(2020)]{mozannar2020defer}
Hussein Mozannar and David Sontag.
\newblock Consistent estimators for learning to defer to an expert.
\newblock In \emph{International Conference on Machine Learning}, pages
  7076--7087. PMLR, 2020.

\bibitem[Orabona and P{\'a}l(2015)]{orabona2015optimal}
Francesco Orabona and D{\'a}vid P{\'a}l.
\newblock Optimal non-asymptotic lower bound on the minimax regret of learning
  with expert advice.
\newblock \emph{arXiv preprint arXiv:1511.02176}, 2015.

\bibitem[Potts and Sammut(2005)]{potts2005incremental}
Duncan Potts and Claude Sammut.
\newblock Incremental learning of linear model trees.
\newblock \emph{Machine Learning}, 61\penalty0 (1):\penalty0 5--48, 2005.

\bibitem[Raghavendra and Yau(2020)]{raghavendra2020list}
Prasad Raghavendra and Morris Yau.
\newblock List decodable learning via sum of squares.
\newblock In \emph{Proceedings of the Fourteenth Annual ACM-SIAM Symposium on
  Discrete Algorithms}, pages 161--180. SIAM, 2020.

\bibitem[Siahkamari et~al.(2020)Siahkamari, Gangrade, Kulis, and
  Saligrama]{siahkamari2020piecewise}
Ali Siahkamari, Aditya Gangrade, Brian Kulis, and Venkatesh Saligrama.
\newblock Piecewise linear regression via a difference of convex functions.
\newblock In \emph{International Conference on Machine Learning}, pages
  8895--8904. PMLR, 2020.

\bibitem[Song et~al.(2016)Song, Kull, Flach, and Kalogridis]{song2016subgroup}
Hao Song, Meelis Kull, Peter Flach, and Georgios Kalogridis.
\newblock Subgroup discovery with proper scoring rules.
\newblock In \emph{Joint European Conference on Machine Learning and Knowledge
  Discovery in Databases}, pages 492--510. Springer, 2016.

\bibitem[Sutton et~al.(2020)Sutton, Boley, Ghiringhelli, Rupp, Vreeken, and
  Scheffler]{sutton2020identifying}
Christopher Sutton, Mario Boley, Luca~M Ghiringhelli, Matthias Rupp, Jilles
  Vreeken, and Matthias Scheffler.
\newblock Identifying domains of applicability of machine learning models for
  materials science.
\newblock \emph{Nature communications}, 11\penalty0 (1):\penalty0 1--9, 2020.

\bibitem[Vershynin(2018)]{vershynin2018hdp}
Roman Vershynin.
\newblock \emph{High-dimensional probability: An introduction with applications
  in data science}, volume~47.
\newblock Cambridge university press, 2018.

\bibitem[Vieth(1989)]{vieth1989piecewise}
Elisabeth Vieth.
\newblock Fitting piecewise linear regression functions to biological
  responses.
\newblock \emph{Journal of applied physiology}, 67\penalty0 (1):\penalty0
  390--396, 1989.

\bibitem[Wang and Witten(1996)]{wang1996induction}
Yong Wang and Ian~H Witten.
\newblock Induction of model trees for predicting continuous classes.
\newblock 1996.

\bibitem[Wang et~al.(2017)Wang, Chi, Che, Chen, Sun, Wang, and
  Wang]{wang2017fasting}
Yunyang Wang, Jingwei Chi, Kui Che, Ying Chen, Xiaolin Sun, Yangang Wang, and
  Zhongchao Wang.
\newblock Fasting plasma glucose and serum uric acid levels in a general
  chinese population with normal glucose tolerance: A u-shaped curve.
\newblock \emph{PLoS One}, 12\penalty0 (6):\penalty0 e0180111, 2017.

\bibitem[Wen et~al.(2017)Wen, Shin, Kang, Yim, Kim, Lee, Lee, Park, Kim, Kweon,
  et~al.]{wen2017association}
Lihui Wen, Min-Ho Shin, Ji-Hyoun Kang, Yi-Rang Yim, Ji-Eun Kim, Jeong-Won Lee,
  Kyung-Eun Lee, Dong-Jin Park, Tae-Jong Kim, Sun-Seog Kweon, et~al.
\newblock Association between grip strength and hand and knee radiographic
  osteoarthritis in korean adults: Data from the dong-gu study.
\newblock \emph{PLoS One}, 12\penalty0 (11):\penalty0 e0185343, 2017.

\bibitem[Wiener and El-Yaniv(2012)]{wiener2012pointwise}
Yair Wiener and Ran El-Yaniv.
\newblock Pointwise tracking the optimal regression function.
\newblock \emph{Advances in Neural Information Processing Systems}, 25, 2012.

\bibitem[Zhang et~al.(2016)Zhang, Li, Liu, Qiao, Zhang, Zhou, Tang, Shen, and
  Chen]{zhang2016stigma}
Chen Zhang, Xiaoming Li, Yu~Liu, Shan Qiao, Liying Zhang, Yuejiao Zhou, Zhenzhu
  Tang, Zhiyong Shen, and Yi~Chen.
\newblock Stigma against people living with hiv/aids in china: does the route
  of infection matter?
\newblock \emph{PloS one}, 11\penalty0 (3):\penalty0 e0151078, 2016.

\end{thebibliography}

\newpage
\onecolumn
\appendix

\section{Geometric Intuition for the Directed Infinity Norm} \label{appendix: directed infty norm}
Here we provide a concrete example of the use of the directed infinity norm in Algorithm~\ref{alg: growing box}. For simplicity, take $U = \{\pm e_i\}_{i=1}^d$, where $e_i$ are the standard basis vectors for $\R^d$. For any vector $x$, we have $\|x\|_{U, \infty} = \max \{x[i], -x[i]\}_{i=1}^d$, where $x[i]$ is the $i$-th component of $x$. Thus we see that $\|x\|_{U,\infty} = \|x\|_\infty$ coincides with the standard $\ell_\infty$ norm in this case. We therefore select the point $x^*$ with the smallest $\ell_\infty$ norm first. If we think of expanding an $\ell_\infty$ ball starting from the origin, $x^*$ is the first point the surface of the ball will come into contact with as it expands.

Suppose WLOG that $u^* = e_1$. This means that the side of the expanding $\ell_\infty$ ball with outward normal in the $e_1$ direction is the side which ``contacted'' $x^*$. Thus any points with $e_1^\T x > e_1^\T x^*$ lie past this face of the expanding box, and therefore cannot possibly constrain any of the other sides of the box. Thus, we no longer need to consider such points.

Similarly, the side of the box which was expanding in the $e_1$ direction is no longer moving outwards. Of the remaining directions of expansion, our goal is to find the next point that a side will come into contact with. By the same logic as before, with $U = \{-e_1\}\cup \{\pm e_i\}_{i=2}^d$, the $\argmax_{u\in U} u^\T x$ tells us which of the \emph{remaining} faces $x$ lies above. Therefore, $\argmin \|x\|_{U,\infty}$ is the point which supports one of the remaining growing faces closest to the origin, i.e., the next point of contact for the expanding box.

\section{MSE vs. Subgroup Size} \label{appendix: mse vs size}
As mentioned in Section~\ref{sec: setup}, \modelname offers a flexible tradeoff between subgroup size and MSE. To implement this tradeoff, we can simply require that the selected region contains at least a proportion $p$ of the validation set. We then select the region with the lowest validation MSE among those regions satisfying this requirement. By varying $p$ between $0$ and $1$, we can smoothly trade off between the size of the selected subgroup (larger $p$) and the MSE on the selected subgroup (smaller $p$).

Figure~\ref{fig: mse vs size} shows the results of this procedure. The $x$-axis shows the fraction of test points included in the selected region, and the $y$-axis shows the test MSE of the model in that region (normalized by the MSE of the baseline model fit to the entire dataset; lower is better). We generated these plots by choosing $p \in \{0.05, \: 0.1, \: 0.2, \: 0.3, \: \ldots, \: 0.9, 1.0\}$ and repeating the experiment across 10 random train/validation/test splits for each dataset. As expected, there is a general positive correlation between the size of the selected group ($x$-axis) and the MSE. 
\begin{figure}
    \centering
    \includegraphics[width=\linewidth]{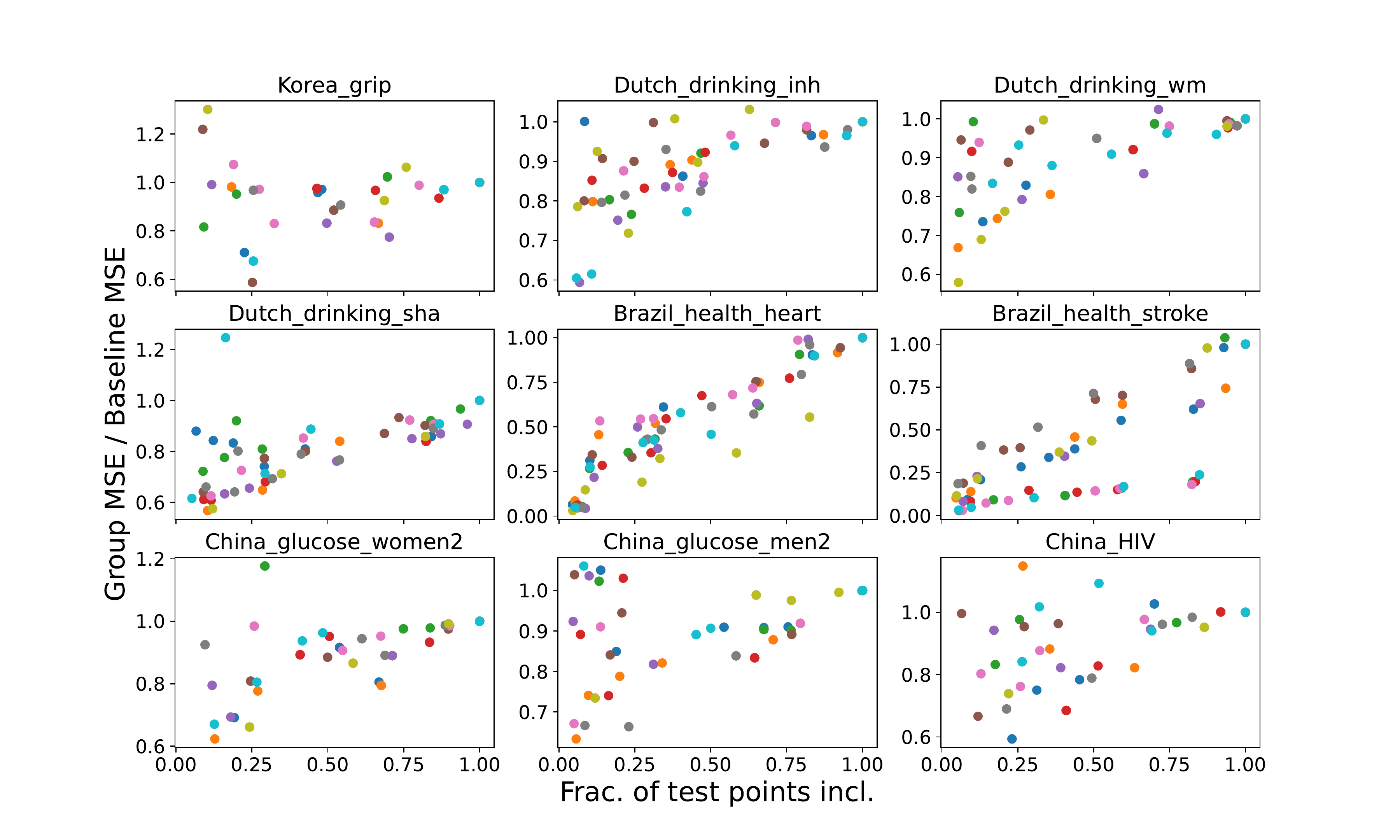}
    \caption{MSE vs. subgroup size selected by \modelname. The $x$-axis shows the fraction of test points included in the selected region. The $y$-axis shows the MSE of the model on the test points in the selected region, normalized by the test MSE of the base model on the whole dataset. (Lower is better.) Different colored points correspond to different random training/validation/test splits on the same dataset. There are 10 random splits in total for each dataset.}
    \label{fig: mse vs size}
\end{figure}

\section{More Experiment Details}
For the experiment in Table~\ref{table: sample size}, we set $R^* = [-1/3, 1/3]^2$ and the bounding region $B = [-1, 1]^2$. We also set $\s_{\textrm{in}} = 0.3$ and $\s_{\textrm{out}} = 5.0$. For \modelname, we did a hyperparameter search over constant rejection thresholds $\rho \in \{2, 4, 8, 16, 32, 64\}$. The core group size was always chosen to be $k=n/20$. Rather than using the variable box growing speeds, we instead added a ``shrinkage'' hyperparameter $\d$: in Algorithm~\ref{alg: growing box}, we only consider points $x \in X_{\textrm{rej}}$ where $x^\T u^* < a^* - \d$. Geometrically, after the growing box ``collides'' with a rejected point, we ``shrink'' that side back by $\d$ opposite to its normal vector. We did a hyperparameter search over $\d \in \{0.1, 0.05, 0.025, 0.01\}$.

For each experiment, we used $20\%$ of the $n$ points as a validation set to select the hyperparameters. 
For \modelname, to select the hyperparameters using the validation set, we used the following procedure. Let $\hat{R}$ be the region selected by a particular setting of the hyperparameters. Let $\hat{\s}$ be an estimate of $\s$ from the core group: 
$$\hat{\s}^2 = \frac{1}{k-d} \sum_{\substack{(x,y) \\ x \in X_{\textrm{core}}}} (\hat{\b}^\T x - y)^2.$$
Let $\hat{q}$ be the $0.9$-quantile of the absolute residuals on the validation set:
$$\hat{q} = \inf \l\{q \: : \: \frac{1}{k} \sum_{\substack{(x,y) \\ x \in X_{\textrm{core}}}} \I\{|\hat{\b}^\T x - y| \leq q\} \geq 0.9\r\}. $$
We selected the hyperparameters which produced the largest region $\hat{R}$ (measured in terms of volume) for which $\hat{q} \leq 3\hat{\s}$.

Lastly, for the LMT method on this experiment, we tuned the tree depth from 1 to twice the dimension.

Code for the experiments can be found at \href{https://github.com/zleizzo/DDGroup}{\texttt{https://github.com/zleizzo/DDGroup}}.

\section{Omitted Proofs} \label{appendix: proofs}
Let $n$ be the total number of points, $k$ the size of the core group.
In what follows, $\|\cdot\|$ denotes the $\ell_2$ norm. We will also sometimes use $x[i]$ to refer to the $i$-th component of the vector $x$.

$X$ will be used to denote the design matrix of a particular group of points (usually a group of $k$ nearest neighbors as considered by the first phase of the algorithm), and $Y$ will denote the  vector of labels corresponding to $X$. We also use the notation $X\cap R^*$; this refers to the matrix of rows of $X$ which are contained in $R^*$.

We use the notation $\cB(x, r)$ to denote the $\ell_2$ ball of radius $r$ centered at $x$. Finally, ``with high probability'' means with probability approaching 1 as $n \rightarrow \infty$ or $k\rightarrow\infty$.

We remark briefly that we select the KNN neighborhood of a point based on the Euclidean norm. This may seem to be a geometric mismatch with the region $R^*$; since $R^*$ is an axis-aligned box, an $\ell_\infty$ neighborhood (which is also an axis-aligned box) might seem more appropriate. Since the $\ell_2$ and $\ell_\infty$ norms are equivalent, this distinction will not make any difference for our theoretical results. Empirically, we also do not notice much change in performance. Thus, we will use the more familiar $\ell_2$ norm for selecting the core group.

\begin{lem} \label{thm: chebyshev}
Let $Z_i$ be independent random variables with uniformly bounded fourth moments. Then $$\P\l(\l|\frac1n \sum_{i=1}^n Z_i - \E Z_i\r| \geq n^{-1/8}\r) = O(n^{-3/2}). $$
\end{lem}
\begin{proof}
This is just a generalization of the standard Chebyshev inequality, and the proof proceeds in the same way. By Markov's inequality, we have
$$ \P\l(\l|\frac1n \sum_{i=1}^n Z_i - \E Z_i\r| \geq t\r)
= \P\l(\l(\sum_{i=1}^n Z_i - \E Z_i\r)^4 \geq n^4 t^4\r)
\leq \frac{\E[(\sum_{i=1}^n Z_i - \E Z_i)^4]}{n^4 t^4}. $$
Expanding $(\sum_{i=1}^n Z_i - \E Z_i)^4$ and taking expectation, by linearity of expectation and independence of the $Z_i$, the only terms which do not vanish are of the form $(Z_i - \E Z_i)^4$ and $(Z_i - \E Z_i)^2 (Z_j - \E Z_j)^2$. There are $O(n^2)$ of all of these terms, each with expectation bounded by $O(1)$, so we obtain
$$ \P\l(\l|\frac1n \sum_{i=1}^n Z_i - \E Z_i\r| \geq t\r) = O\l(\frac{1}{n^2 t^4}\r). $$
Substituting $t = n^{-1/8}$ completes the proof.
\end{proof}

\begin{lem} \label{thm: knn in Rstar}
There exists a constant $c_1 > 0$ (which can depend on $R^*, c, C, R$) such that if $k\leq c_1n$, there exists a set of $k$ nearest neighbors of some point in $X$ which is contained in $R^*$ with high probability.
\end{lem}
\begin{proof}
By Assumption~\ref{assumption: nonempty}, $R^*$ has nonempty interior. Thus there exists a point $\bar{x} \in R^*$ and a radius $r>0$ such that $\cB(\bar{x}, r) \subseteq R^*$. Consider $\cB(\bar{x}, r/4) \subseteq R^*$. By Assumption~\ref{assumption: density}, $\P(x \in \cB(\bar{x}, r/4)) \geq c_f \cdot \vol(\cB(\bar{x}, r/4)) \equiv p_1 = \Omega(1)$. By Hoeffding's inequality, we have
$$ \P\l(\sum_{i=1}^n \I\{x_i \in \cB(\bar{x}, r/4)\} \leq p_1n - t\r) \leq e^{-2t^2 / n} \hspace{.15in} \Longrightarrow \hspace{.15in} \sum_{i=1}^n \I\{x_i \in \cB(\bar{x}, r/4)\} \geq p_1n - \sqrt{\frac{n \log n}{2}}$$
with probability at least $1-1/n$. In particular, since $p_1 = \Omega(1)$, for $n$ large enough we have that $\cB(\bar{x}, r/4)$ contains at least \emph{one} point $x_{i^*}$ with high probability.

Next, consider $\cB(\bar{x}, r/2)$. Setting $p_2 = c_f \cdot \vol(\cB(\bar{x}, r/2))$, the same argument as above shows that
$$ \sum_{i=1}^n \I\{x_i \in \cB(\bar{x}, r/2) \geq p_2n - \sqrt{\frac{n\log n}{2}} \hspace{.25in} \textrm{w.p.} \geq 1-1/n. $$
In particular, if we take $c_1 = p_2/2 = \Omega(1)$, then $k \leq c_1n$ implies that there are at least $k$ points contained in $\cB(\bar{x}, r/2)$ with high probability for $n$ large enough. We claim that in this event, the KNN of $x_{i^*}$ is contained in $\cB(\bar{x}, r)\subseteq R^*$. This is because for $x' \not\in \cB(\bar{x}, r)$, we have
$$ \|x' - x_{i^*}\| \geq \|x' - \bar{x}\| - \|\bar{x} - x_{i^*}\| > r - r/4 = 3r/4.$$
However, for a point $x' \in \cB(\bar{x}, r/2)$, we have
$$ \|x' - x_{i^*}\| \leq \|x' - \bar{x}\| + \|\bar{x} - x_{i^*}\| \leq r/2 + r/4 = 3r/4.$$
Thus with probability approaching 1 as $n\rightarrow\infty$, there exists a set of $k$ nearest neighbors contained in $R^*$ provided that $k \leq c_1n$.
\end{proof}
Henceforth, we will assume that $k \leq c_1n$.

\begin{lem} \label{thm: slabs}
Let $u$ be any unit vector and define $S_{u,r} = \{x \in X \: : \: |x^\T u|\leq r\}$. With probability at least $1-1/n$, we have $|S_{u,r}| \leq c_2rn + \sqrt{\frac{n\log n}{2}}$ for some constant $c_2$.
\end{lem}
\begin{proof}
Since $S$ is bounded, we have that $\vol(S_{u,r} \cap S) \leq c_2' r$ for some constant $c_2'$ (which depends on the size of $S$). By Assumption~\ref{assumption: density}, since the density of $x$ is bounded, this means that $\P(x \in S_{u,r}) \leq c_2 r$ for some constant $c_2$. Thus by Hoeffding's inequality, we have that
$$ |\{x \: : \: |u^\T x| < r\}| \leq c_2nr + \sqrt{\frac{n\log n}{2}} \hspace{.25in} \textrm{w.p.} \geq 1-1/n$$
as desired.
\end{proof}

\begin{lem} \label{thm: lambda min}
If $k = \Omega(n)$, then with high probability,
\emph{every} group of $k$ points $X$ has $\smin(\frac1k X^\T X) = \Omega(1)$.
\end{lem}
\begin{proof}
First, consider a fixed $\|u\|=1$. For any group of $k$ points, let $X$ be the associated data matrix and define $A = \frac1k X^\T X$. By Lemma~\ref{thm: slabs}, we have
\begin{align*}
    u^\T A u &= \frac1k \sum_{i=1}^k (x_i^\T u)^2 \\[5pt]
    &\geq \frac1k \sum_{i \: : \: |x_i^\T u| \geq r} (x_i^\T u)^2 \\[5pt]
    &\geq \frac{1}{k} \l(k - c_2rn - \sqrt{\frac{n\log n}{2}}\r) r^2.
\end{align*}
Let $c_3 > 0$ be a constant such that $k \geq c_3 n$ and define $r = c_3/2c_2$. We obtain the lower bound
$$ u^\T A u \geq \l(1 - \frac{c_2rn}{c_3n} - \frac{\sqrt{\frac{n\log n}{2}}}{c_3n}\r) \frac{c_3^2}{4c_2^2} = \Omega(1) $$
for $n$ large enough. Let $c_4 = \Omega(1)$ be a lower bound on this quantity for large $n$.

Observe that by the fact that the $x_i$ are bounded, we trivially have $\|A\| \leq B^2:$
$$ u^\T A u = \frac1k \sum_{i=1}^k (u_i^\T x)^2 \leq \frac1k \sum_{i=1}^k \|x_i\|^2 \leq B^2. $$

Next, let $E = \{u_i\}_{i=1}^N$ be an $\e$-net for the unit sphere, and note that we can take $N \leq (3/\e)^d$. By applying a union bound over $E$, we have that $u_i^\T (\frac1k X^\T X) u_i \geq c_4$ for all $i$ with probability at least $1 - N/n$. Let $\|u\|=1$ be arbitrary and choose $u_i$ such that $\|u-u_i\|\leq \e$. Then we have
\begin{align*}
    u^\T A u &= u_i^\T A u_i + (u-u_i)^\T A u + u_i^\T A (u-u_i) \\[5pt]
    &\geq c_4 - \|A\|\|u-u_i\|\|u\| - \|A\|\|u_i\|\|u-u_i\| \\[5pt]
    &\geq c_4 - 2B^2\e.
\end{align*}
Thus if we take $\e = c_4/4B^2$, we have that $u^\T A u \geq c_4/2 = \Omega(1)$ for all $\|u\|=1$. This occurs with probability at least $1-(3/\e)^d/n = 1-O(1/n)$, i.e. with high probability. (Note: We only need that $|S_{u_i, r}|$ is bounded according to Lemma~\ref{thm: slabs} for each $u_i$, then these inequalities hold simultaneously for \emph{all} groups of $k$ points. In particular, we do not need to take a union bound over the groups of neighboring points.)
\end{proof}

\begin{lem} \label{thm: hat matrix}
Suppose that $k \geq d$ and $X \in \R^{k \times d}$ has full rank. Then there exist unit vectors $u_i \in \R^n$, $i=1,\ldots,d$ such that $H \equiv X (X^\T X)^{-1}X^\T = \sum_{i=1}^d u_i u_i^\T$.
\end{lem}
\begin{proof}
Let $X = U\Sigma V^\T$ be the SVD of $X$, where $\Sigma \in \R^{k\times d}$ has diagonal entries $\s_i$. Since $X$ has full rank, $\s_i > 0$ for all $i=1,\ldots,d$. Define $\Sigma^{-2} = \textrm{diag}(\s_i^{-2}) \in \R^{d\times d}$. We have
\begin{align*}
    H &= (U \Sigma V^\T) (V\Sigma^{-2}V^\T) V \Sigma^\T U^\T \\[5pt]
    &= U \Sigma (\Sigma^{-2}) \Sigma^\T U^\T \\
    &= \sum_{i=1}^d u_i u_i^\T,
\end{align*}
where $u_i \in \R^k$ are the columns of $U$ (i.e., the left singular vectors of $X$).
\end{proof}

\begin{lem} \label{thm: anderson}
Let $Y \sim \cN(0, \Sigma)$, where $\Sigma \in \R^{m\times m}$ has singular values $\s_1 \geq \cdots \geq \s_m$. Let $\mu \in \R^m$ be independent of $Y$. If $\s_m \geq \s$, then
$$ \P\l(\|\mu + Y\| \leq \s\sqrt{m} - t\r) \leq \P\l(\|Y\|\leq \s\sqrt{m} - t\r) \leq 2\exp\l(-Ct^2/\s^2\r) $$
for some universal constant $C$.
\end{lem}
\begin{proof}
This follows directly from Lemmas 9 and 10 in \cite{izzo2022stateful}.
\end{proof}

\goodcore*
\begin{proof}
By Lemma~\ref{thm: knn in Rstar}, there exists a group of $k$ points contained in $R^*$ with high probability. For these points, we have
\begin{align}
\min_\b \frac1k \sum_{i=1}^k (x_i^\T \b - y_i)^2 &\leq \frac1k \sum_{i=1}^k (x_i^\T \b^* - y_i)^2 \\[5pt]
&= \frac1k \sum_{i=1}^k (y_i - \E[y_i | x_i])^2 \nn \\[5pt]
&\leq \s^2 + k^{-1/8} \hspace{.15in} \textrm{w.p.} \geq 1-O(k^{-3/2}). \label{eq: inside Rstar ub}
\end{align}
On the other hand, let $\d > 0$ be fixed and consider a group of $k$ points at least $m \geq \d k$ of which are \emph{not} in $R^*$. WLOG assume that the first $i=1,\ldots, m$ points lie outside $R^*$ and the remaining $k-m$ points are in $R^*$. Let $\mu_i = \E[y_i|x_i]$ and $z_i = y_i - \mu_i$. Then we have
\begin{align}
    \min_\b \frac1k \sum_{i=1}^k (x_i^\T \b - y_i)^2 &= \min_\b \frac1k \sum_{i=1}^k (x_i^\T \b - (\mu_i + z_i))^2 \nn \\[5pt]
    &\geq \frac1k \min_\b \sum_{i=1}^m (x_i^\T \b - (\mu_i + z_i))^2 + \frac1k \min_\b \sum_{i=m + 1}^k (x_i^\T \b - (\mu_i + z_i))^2 \nn \\[5pt]
    &= \frac1k \l\|(I-H_1)(\boldsymbol{\mu}_1 + Z_1)\r\|^2 + \frac1k\l\|(I-H_2)(\boldsymbol{\mu}_2 + Z_2)\r\|^2, \label{eq: outside Rstar lb 1}
\end{align}
where we define $X_1$ as the data matrix for the $m$ points not in $R^*$ and $H_1 = X_1 (X_1^\T X_1)^{-1} X_1^\T$. We define $\boldsymbol{\mu}_1 = (\mu_1,\ldots,\mu_m)^\T$ and $Z_1 = (z_1,\ldots,z_m)^\T$. $X_2, H_2$, $\boldsymbol{\mu}_2$, and $Z_2$ are defined similarly for the $k-m$ points in $R^*$.

By Lemma~\ref{thm: hat matrix}, we can write $H_1 = \sum_{i=1}^d u_i u_i^\T$ for some orthonormal $u_i \in \R^m$. Extend to an orthonormal basis $u_1,\ldots,u_m$ for $\R^m$, and note that $I = \sum_{i=1}^m u_i u_i^\T$. It follows that $I - H_1 = \sum_{i=d+1}^m u_i u_i^\T$, and therefore $$(I - H_1) Z_1 = (u_{d+1}^\T Z_1) u_{d+1} + \cdots + (u_m^\T Z_1) u_m.$$
In particular, $(I-H_1)Z_1 \sim \cN(0, I_{m-d})$. A similar calculation shows that $(I - H_2) Z_2 \sim \cN(0, I_{n-m-d})$. (Here we assume that $m \leq n-d$; if not, we can just use the fact that the second term in \eqref{eq: outside Rstar lb 1} is nonnegative.) By applying Lemma~\ref{thm: anderson}, we obtain
\begin{align}
    &\P\l[ \min_\b \sum_{i=1}^k (x_i^\T \b - y_i)^2 \leq \s_0^2\l(\sqrt{m-d} - \sqrt{\frac{\log k}{C}}\r)^2 + \s^2 \l(\sqrt{k-m-d} - \sqrt{\frac{\log k}{C}} \r)^2 \r] \nn \\[5pt]
    &\leq \P\l[\l\|(I-H_1)(\boldsymbol{\mu}_1 + Z_1)\r\| \leq \s_0\l(\sqrt{m-d} - \sqrt{\frac{\log k}{C}} \r)\r] + \P\l[ \l\|(I-H_2)(\boldsymbol{\mu}_2 + Z_2)\r\| \leq \s\l(\sqrt{k-m-d} - \sqrt{\frac{\log k}{C}} \r) \r] \nn \\[5pt]
    &\leq 4/k. \nn
\end{align}
The final inequality is obtained by applying Lemma~\ref{thm: anderson} to each of the two preceding terms. It follows that with high probability (at least $1-4/k$), we have that
\begin{align}
    \min_\b \frac1k \sum_{i=1}^k (x_i^\T \b - y_i)^2 &\geq \frac{\s_0^2 m}{k} + \frac{\s^2 (k-m)}{k} - o(1) \nn \\
    &\geq \s^2 + (\s_0^2 - \s^2) \d - o(1). \label{eq: outside Rstar lb 2}
\end{align}
For any constant $\d>0$ and $k$ large enough, \eqref{eq: outside Rstar lb 2} will be strictly greater than the upper bound in \eqref{eq: inside Rstar ub}. It follows that with high probability, all but $o(k)$ points in the selected core group will belong to $R^*$.
\end{proof}

\begin{lem} \label{thm: inverses are close}
Let $X$ be the group of $k$ points selected by Algorithm~\ref{alg: core group}, and define $\widetilde{X} = X \cap R^*$. Then we have that $\|(\frac1k X^\T X)^{-1} - (\frac1k \widetilde{X}^\T \widetilde{X})^{-1}\| = o(1)$ with high probability.
\end{lem}
\begin{proof}
Let $A = \frac1k X^\T X$ and $B = \frac1k \widetilde{X}^\T \widetilde{X}$. By Lemma~\ref{thm: lambda min}, $\smin(A)\geq c_4 = \Omega(1)$ with high probability. By Lemma~\ref{thm: good core whp}, $\widetilde{X}$ contains all but $o(k)$ of the selected points. Also recall that by our assumptions, all of the $x$ are bounded. By Weyl's inequality, we have $$\smin(B) \geq \smin(A) - \smax\l( \frac1k\sum_{x\in X \setminus R^*} xx^\T \r) = \Omega(1) - o(1) = \Omega(1).$$
In particular, this means that $\|B^{-1}\| = \smin(B)^{-1} = O(1)$. Finally, since $A^{-1} - B^{-1} = A^{-1}(B-A)B^{-1}$, we have
$$ \|A^{-1} - B^{-1}\| \leq \|A^{-1}\| \|B - A\| \|B^{-1}\| = O(1)\: \cdot \: o(1) \: \cdot \: O(1) = o(1).$$
\end{proof}

\begin{lem} \label{thm: beta hat bdd}
$\|\hat{\b}\| = O(1)$ with high probability.
\end{lem}
\begin{proof}
We know from the proof of Lemma~\ref{thm: good core whp} that there exists a group of $k$ nearest neighbors contained in $R^*$, and that for such a group, the training MSE is at most $\s^2 + o(1)$ with high probability. Thus, since the core group has the minimum training MSE, we must have
\begin{align}
    \s^2 + o(1) &\geq \frac1k \|X\hat{\b} - Y\|^2 \nn \\[5pt]
    &\geq \frac1k \|\tX\hat{\b} - \tY\|^2 \nn \\[5pt]
    &\geq \frac1k \|\tX\b^* - \tY\|^2 - \frac2k \|\tX\b^* - \tY\|\|\tX(\hat{\b}-\b^*)\| + \frac1k \|\tX(\hat{\b}-\b^*)\|^2 \nn \\[5pt]
    &\geq \s^2 - o(1) + \frac1k \l( \|\tX(\hat{\b}-\b^*)\|^2 - 2\|\tX\b^* - \tY\|\|\tX(\hat{\b}-\b^*)\| \r). \label{eq: beta hat bdd 1}
\end{align}
Inequality \eqref{eq: beta hat bdd 1} holds because $\tX, \tY$ contain $k-o(k)$ points and by roughly the same logic used to obtain the upper bound in \eqref{eq: inside Rstar ub}. It follows that
\begin{equation} \label{eq: beta hat bdd 2}
\frac1k \l( \|\tX(\hat{\b}-\b^*)\|^2 - 2\|\tX\b^* - \tY\|\|\tX(\hat{\b}-\b^*)\| \r) = o(1).
\end{equation}

Again using the logic from \eqref{eq: inside Rstar ub}, we have that $\|\tX\b^* - \tY\| = \s\sqrt{k} + o(\sqrt{k})$. Combining this with equation \eqref{eq: beta hat bdd 2} therefore shows that $\|\tX(\hat{\b}-\b^*)\| = O(\sqrt{k})$ (this follows from a simple application of the quadratic formula), or equivalently $\|\tX(\hat{\b}-\b^*)\|^2 = O(k)$.

Finally, observe that $$\|\tX(\hat{\b}-\b^*)\|^2 = (\hat{\b}-\b^*)\tX^\T \tX(\hat{\b}-\b^*) \geq \smin(\tX^\T \tX) \|\hat{\b}-\b^*\|^2.$$ By the proof of Lemma~\ref{thm: inverses are close}, we know that $\smin(\frac1k \tX^\T \tX) = \Omega(1)$, so $\smin(\tX^\T \tX) = \Omega(k)$. Thus we have
$$ O(k) = \|\tX(\hat{\b} - \b^*)\|^2 \geq \Omega(k) \|\hat{\b}-\b^*\|^2 \hspace{.15in} \Longrightarrow \hspace{.15in} \|\hat{\b} - \b^*\| = O(1). $$
Since $\|\b^*\|$ is a constant, we conclude that $\|\hat{\b}\| = O(1)$ by the triangle inequality.
\end{proof}

\begin{lem} \label{thm: XTY is close}
$\l\| \frac1k X^\T Y - \frac1k \tX^\T \tY \r\| = o(1)$ with high probability.
\end{lem}
\begin{proof}
We begin with the same observation used to prove Lemma~\ref{thm: beta hat bdd}, namely that the training MSE for the core group must be upper bounded by $\s^2 + o(1)$ with high probability. By Assumption~\ref{assumption: density}, $\|x\| = O(1)$, and by Lemma~\ref{thm: beta hat bdd}, $\|\hat{\b}\| = O(1)$ with high probability. Let $C$ be a constant such that $\|x\|\|\hat{\b}\| \leq C$. We then have
\begin{align}
\frac1k \|X\hat{\b} - Y\|^2 &\geq \frac1k \|\tX\hat{\b} - \tY\|^2 + \frac1k \sum_{(x, y)\: : \: x\not\in R^*}(x^\T \hat{\b} - y)^2 \nn \\[5pt]
&\geq \s^2 - o(1) + \frac1k \sum_{\substack{(x, y) : x\not\in R^* \\ |y| \geq 2C+1}} |y|(|y|-2\|x\|\|\hat{\b}\|) \nn \\[5pt]
&\geq \s^2 - o(1) + \frac1k \sum_{\substack{(x, y) : x\not\in R^* \\ |y| \geq 2C+1}} |y|. \nn
\end{align}
It therefore follows that 
\begin{equation} \label{eq: XTY is close 2}
\frac1k \sum_{\substack{(x, y) : x\not\in R^* \\ |y| \geq 2C+1}} |y| = o(1).
\end{equation}

Since $\|x\| \leq B$, we now have
\begin{align*}
\l\| \frac1k X^\T Y - \frac1k \tX^\T \tY \r\| &\leq \frac1k\sum_{(x, y) : x\not\in R^*} \|x\||y| \\[5pt]
&\leq  \frac1k \sum_{\substack{(x, y) : x\not\in R^* \\ |y| < 2C+1}} B(2C+1) + \frac1k\sum_{\substack{(x, y) : x\not\in R^* \\ |y| \geq 2C+1}} B|y| \\[5pt]
&= o(1).
\end{align*}
The final conclusion holds by applying Lemma~\ref{thm: good core whp} and equation \eqref{eq: XTY is close 2} to the two terms in the previous line.
\end{proof}

\nobadrej*
\begin{proof}
First, we will show that $\|\b^* - \hat{\b}\| = o(1)$ with high probability. Let $\tX = X \cap R^*$ denote the data matrix for the core points which belong to $R^*$, and let $\tY$ denote the response vector corresponding to these points. We use the identity
$$ \hat{\b} - \b^* = \underbrace{\l[ \l(\frac1k X^\T X\r)^{-1} - \l(\frac1k \tX^\T \tX\r)^{-1} \r] \l(\frac1k X^\T Y\r)}_{\mathrm{(I)}} +  \underbrace{\l(\frac1k \tX^\T \tX\r)^{-1} \l[ \frac1k X^\T Y - \frac1k \tX^\T \tY \r]}_{\mathrm{(II)}} + \underbrace{\l(\frac1k \tX^\T \tX\r)^{-1}\l(\frac1k \tX^\T \tY\r) - \b^*}_{\mathrm{(III)}}.$$

\paragraph{Term (I)} By Lemma~\ref{thm: inverses are close}, $\| (\frac1k X^\T X)^{-1} - (\frac1k \tX^\T \tX)^{-1} \| = o(1)$. By Lemma~\ref{thm: XTY is close} and an application of the triangle inequality, $\|\frac1k X^\T Y\| = O(1)$, so term (I) is $o(1)$.

\paragraph{Term (II)} By the proof of Lemma~\ref{thm: inverses are close}, $\|(\frac1k \tX^\T \tX)^{-1}\| = O(1)$. By Lemma~\ref{thm: XTY is close}, $\|\frac1k X^\T Y - \frac1k \tX^\T \tY\| = o(1)$, so term (II) is $o(1)$.

\paragraph{Term (III)} Let $k'$ be the number of points in $\tX$ (so $k' = k - o(k)$) and WLOG assume that the points in $\tX$ are the first $k'$ points $x_1,\ldots, x_{k'}$. We have $\tY = \tX\b^* + E$, where $E = (\e_i)_{i=1}^{k'}$ is the vector of error terms and $\e_i \sim \cN(0, \s^2)$. Define $\tilde{\b} = (\frac{1}{k'} \tX^\T \tX)^{-1} (\frac{1}{k'} \tX^\T \tY)$ and note that this is still equal to the first term in (III). It follows that
$$\tilde{\b} = \b^* + \l(\frac{1}{k'} \tX^\T \tX\r)^{-1}\sum_{i=1}^{k'} \e_i x_i,$$
This implies that
$$\|\hat{\b} - \b^*\| \leq \l\|\l(\frac{1}{k'}\tX^\T \tX\r)^{-1}\r\| \l\|\frac{1}{k'} \sum_{i=1}^k \e_i x_i\r\| = \s \smin^{-1} \l\|\frac{1}{k'} \sum_{i=1}^{k'} g_i x_i\r\|,$$
where $g_i\iid \cN(0,1)$ and $\smin = \smin(\frac1k \tX^\T \tX)$. It remains to bound $\l\|\frac{1}{k'} \sum_{i=1}^{k'} g_i x_i\r\|$ with high probability. Observe that
$$ \P\l(\l\|\frac{1}{k'} \sum_{i=1}^{k'} g_i x_i\r\| \geq t\r) \leq \sum_{j=1}^d \P\l( \l|\frac{1}{k'}\sum_{i=1}^{k'} g_i x_{ij}\r| \geq \frac{t}{\sqrt{d}}\r). $$
Standard Gaussian concentration results (see e.g. \citep{vershynin2018hdp}) show that the RHS is bounded by $2d\exp\l(\frac{-c k' t^2}{d}\r)$ for some universal constant $c$. Setting this bound equal to $1/k'$ and solving for $t$, we see that $\|\hat{\b} - \b^*\| \leq C\s \smin^{-1} \sqrt{\frac{d \log(2dk')}{k'}}$ with probability at least $1-1/k'$, i.e. with high probability. (Here $C$ is another universal constant.) Since $\smin = \Omega(1)$, we have $\|\hat{\b}-\b^*\| = o(1)$.

Next, we look at $|y_i - x_i^\T \hat{\b}|$ for a point $x_i\in R^*$. In this case, applying the triangle inequality and Cauchy-Schwarz, we have
$$ |y_i - x_i^\T \hat{\b}| = |x_i^\T \b^* - x_i^\T \hat{\b} + \e_i| \leq \|\b^*-\hat{\b}\|\|x_i\| + |\e_i|. $$
Since our dataset contains $n$ points, there are at most $n$ points in $R^*$. Again by standard Gaussian concentration results and a union bound, we have that $|\e_i| \leq \s \sqrt{2\log \frac{2n}{\a}}$ for all $x_i \in R^*$ simultaneously with probability at least $1-\a$. Thus we have that
$$|y_i - x_i^\T \hat{\b}| \leq \s\sqrt{2\log \frac{2n}{\a}} + o(1)$$
for all $x_i \in R^*$ with probability at least $1-\a-o(1)$. Setting $\a = 1/n$ and adjusting the constants slightly to account for the $o(1)$ term, we see that
$$|y_i - x_i^\T \hat{\b}| \leq 2.1\s\sqrt{\log n}$$
with high probability and for large enough $n$, as desired.
\end{proof}

\begin{lem} \label{thm: avg is in R*}
With high probability, the average of the core point feature vectors belongs to $R^*$.
\end{lem}
\begin{proof}
In this proof, we will use the fact that $R^*$ is an axis-aligned box, but we note that this is just for ease of exposition and the results extend readily to the case when $R^*$ is a general convex body.

Let $R^* = \prod_{i=1}^d [a_i, b_i]$ with $a_i < b_i$ for each $i$ and define $\partial R^*_\e = \prod_{i=1}^d [a_i + \e, b_i - \e]$ for $\e < \min_i (b_i - a_i)/2$. (That is, $\partial R^*_\e$ consists of those points in $R^*$ which are at most $\e$ away from the boundary of $R^*$.) A direct calculation shows that $\vol(\partial R^*_\e) = O(\e)$. By the same logic as in Lemma~\ref{thm: slabs}, it follows that there exists an $\e = \Omega(1)$ such that $\partial R^*_\e$ contains at most $m \leq c_3 n/2$ points with high probability, where here $c_3$ is a constant such that $k \geq c_3n$.

Let $\bar{x}$ be the average of the core group feature vectors. We will show that $\bar{x}[i] \leq b_i$. A nearly identical argument will show that all of the components of $\bar{x}$ satisfy the constraints required to belong to $R^*$. Observe that
\begin{align}
    \bar{x}[i] &= \frac1k \l( \sum_{x \in R^* \setminus \partial R^*_\e} x[i] + \sum_{x \in \partial R^*_\e} x[i] + \sum_{x \not\in R^*} x[i] \r) \nn \\[5pt]
    &\leq \frac1k \Large( (k - m - o(k)) (b_i - \e) + m b_i + o(k) \Large)  \label{eq: center 1} \\[5pt]
    &= b_i + (\frac{m}{k} - 1) \e + o(1) \nn \\[5pt]
    &\leq b_i - \e/2 + o(1) \label{eq: center 2} \\[5pt]
    &\leq b_i. \nn
\end{align}
Inequality \eqref{eq: center 1} follows from the fact that the features $x$ (and hence each component $x[i]$) are bounded, and the fact that at most $o(k)$ points in the core group are not in $R^*$ by Lemma~\ref{thm: good core whp}. Inequality \eqref{eq: center 2} holds because $m \leq c_3n/2$ and $k \geq c_3n$. This completes the proof.
\end{proof}

\begin{lem} \label{thm: E abs gauss}
Suppose that $Y_i \sim \cN(0, \s_i^2)$ are independent. Then $\P(\max|Y_i-a_i| \leq t) = \P(\max|Y_i|\leq t)$ and consequently $\E\max|Y_i-a_i| \geq \E\max|Y_i|$ for any constants $a_i$.
\end{lem}
\begin{proof}
Lemma~\ref{thm: anderson} implies that $\P(|Y_i-a_i| \leq t) \leq \P(|Y_i| \leq t)$ for all $i, t$. Furthermore, we have that
$$ \P(\max|Y_i-a_i| \leq t) = \prod_i \P(|Y_i-a_i|\leq t) \leq \prod_i \P(|Y_i| \leq t) = \P(\max |Y_i|\leq t). $$
Integrating by parts shows that
$$ \E\max|Y_i-a_i| = \int_0^\infty \P(\max|Y_i-a_i|\geq t)\, dt \geq \int_0^\infty \P(\max|Y_i|\geq t)\, dt = \E\max|Y_i|. $$
\end{proof}

\begin{lem} \label{thm: gaussian max}
Let $Y_i \sim \cN(0, \s_i^2)$ with $\s_i^2 \geq \s^2$ for all $i$. Then for any constants $a_i$ and $m$ large enough, we have the following inequality:
$$ \E\l[\max_{i=1}^m |Y_i - a_i|\r] \geq 0.12 \s \sqrt{\log m}.$$
Furthermore, we have that $\max_{i=1}^m |Y_i - a_i| = \Omega(\s\sqrt{\log m})$ with high probability as $m\rightarrow\infty$.
\end{lem}
\begin{proof}
Note that by Lemma~\ref{thm: E abs gauss}, it suffices to show the result for $a_i = 0$. Next, observe that since $\s^2 / \s_i^2 \leq 1$ for all $i$, we have
$$ \E\l[\max_{i=1}^m |Y_i|\r] \geq \E\l[\max_{i=1}^m  \frac{\s}{\s_i}|Y_i|\r] \geq \E\l[\max_{i=1}^m Z_i\r], $$
where $Z_i \iid \cN(0, \s^2)$. By Theorem 3 of \cite{orabona2015optimal}, $\E\l[\max_{i=1}^m Z_i\r] \geq 0.13 \s \sqrt{\log m} - 0.7\s \geq 0.12\s\sqrt{\log m}$ for large enough $m$.

Next, we will show that $\P(\max |Y_i| \leq t)$ is decreasing in $\s_i$. We have
\begin{align}
    \P(\max |Y_i| \leq t) &= \prod_{i=1}^m \P(|\s_i Z_i| \leq t), \hspace{.15in} Z_i\iid \cN(0, 1) \nn \\
    &= \prod_{i=1}^m \P(|Z_i| \leq t/\s_i). \label{eq: no sigma ub needed}
\end{align}
Since $\{|Z_i|\leq t/\s_i\} \subseteq \{|Z_i|\leq t/\s_i'\}$ when $\s_i \geq \s_i'$, the terms in \eqref{eq: no sigma ub needed} are decreasing in $\s_i$.

Next, suppose that $\s_i = \s$ for all $i$. We will show that $\V\l(\max_{i=1}^m |Y_i - a_i|\r) \leq \s^2$. By homogeneity, it suffices to show this inequality for $\s=1$. Define $f(Y_1,\ldots,Y_m) = \max_{i=1}^m |Y_i|$. By the Gaussian Poincar\'{e} inequality (see e.g. \cite{chatterjee2014superconcentration}, pg. 47), we have that
$$ \V\l(\max_{i=1}^m |Y_i|\r) \leq \sum_{i=1}^n \E|\partial_i f(Y)|^2. $$
We have $\partial_i f(Y) = \mathrm{sign}(Y_i) \I\{|Y_i| = \max_j |Y_j|\}$ almost everywhere, so
$$ \E|\partial_i f(Y)|^2 = \P(|Y_i| = \max_j |Y_j|\}) = 1/m $$
for all $i$. It follows that $\V(\max_{i=1}^m |Y_i|) \leq 1$. By Chebyshev's inequality, we therefore have that
$$ \P\l( \max_{i=1}^m |Y_i| \leq 0.12 \s \sqrt{\log m} - \s t \r) \leq 1/t^2. $$
In particular, we can take $t = 0.06\sqrt{\log m}$, then $\max_{i=1}^m |Y_i| = \Omega(\s\sqrt{\log m})$ with probability at least $1 - O(1/\log m) = 1 - o(1)$, i.e. with high probability.
\end{proof}

\main*
\begin{proof}
By Lemma~\ref{thm: avg is in R*}, the average of the core group points $\bar{x}$ (and therefore the point from which we begin growing the box in Algorithm~\ref{alg: growing box}) lies in the interior of $R^*$. Let $\partial R^*$ denote the boundary of $R^*$. For each $j=1,\ldots,d$, denote by $\partial R^*_{j, +}$ the face of $\partial R^*$ which upper bounds the $j$-th dimension, and let $\partial R^*_{j, -}$ be the opposite face which lower bounds the $j$-th dimension. Let $s_j^\pm = d(\bar{x}, \partial R^*_{j, \pm})$ be the distance from the center to the appropriate face of $R^*$. Note that Algorithm~\ref{alg: growing box} with these speeds and this center is equivalent to running the algorithm from the origin and with uniform speeds, after shifting the data so that $\bar{x}$ lies at the origin and then rescaling each axis by $s_j^\pm$. In this case, $R^*$ is transformed into a $\ell_\infty$ ball of radius 1 centered at the origin.

By Lemma~\ref{thm: no bad rej}, $R^*$ contains no rejected points with high probability. (Note that the transformations we performed above preserve this fact.) Since the region returned by Algorithm~\ref{alg: growing box} returns a region which contains the largest centered $\ell_\infty$ ball with no rejected points in it, and $R^*$ is a centered $\ell_\infty$ ball with no rejected points, we must have $R^* \subseteq \hat{R}$ as desired.

Since we have assumed that $R^*$ is an axis-aligned box, we can write $R^* = \{x \: | \: a_j < x_j < b_j\}$. Fix $\e > 0$ and let 
$$\partial R^*_{\e, j, +} = \{ x \: | \: b_j \leq x_j \leq b_j + \e, \: \ell_m < x_m < u_m, m \neq j\}$$
$$\partial R^*_{\e, j, -} = \{ x \: | \: a_j - \e \leq x_j \leq a_j, \: \ell_m < x_m < u_m, m \neq j\}.$$
(These are just the sets of points which are at most $\e$ ``above'' the upper dimension $j$ face of $R^*$ and ``below'' the lower dimension $j$ face of $R^*$, respectively.)

By the same logic as in the proof of Lemma~\ref{thm: slabs}, there is some constant $c_6>0$ (which can depend on $R^*$) such that at least $c_6\e n$ points lie in $\partial R^*_\e$ with high probability. Take $\e = n^{-1/2}$ and apply Lemma~\ref{thm: gaussian max} to the $c_6\e n$ points in $\partial R^*_{\e, j, \pm}$. We see that
$$\max_{x_i \in \partial R^*_{\e, j, \pm}} |x_i^\T \hat{\b} - y_i| \geq 0.06\s_0\sqrt{\frac12 \log c_6 n}$$
with high probability. Since $\s_0 > 50\s$, for $n$ large enough we have
$$ \max_{x_i \in \partial R^*_{\e, j, \pm}} |x_i^\T \hat{\b} - y_i| \geq 0.06\s_0\sqrt{\frac12 \log c_6 n} > 2.1\s\sqrt{\log \e n}. $$
The means that Algorithm~\ref{alg: growing box} will stop growing the $(j,\pm)$ side of $\hat{R}$ at some point in $\partial R^*_{\e, j, \pm}$. It follows that $\hat{R} \subseteq R^*_\e$ with $\e = n^{-1/2}$. This completes the proof.
\end{proof}

\end{document}